\documentclass[journal]{IEEEtran}
\usepackage{graphicx}
\usepackage{amssymb}
\usepackage{amsfonts}
\usepackage{mathrsfs}
\usepackage{amsmath}
\usepackage{algorithm}
\usepackage{algorithmic}
\usepackage{multirow}
\usepackage{booktabs}
\usepackage[table]{xcolor}
\usepackage{cite}
\usepackage{setspace}
\definecolor{mygray}{gray}{.9}
\definecolor{babyblueeyes}{rgb}{0.7, 0.8, 1}
\newtheorem{theorem}{Theorem}[section]
\newtheorem{proposition}[theorem]{Proposition}
\renewcommand{\algorithmicrequire}{\textbf{Input:}}

\newenvironment{proof}{\hspace{0ex}\textsc{Proof}.\hspace{1ex}}{\hfill$\blacksquare$\newline}
\ifCLASSINFOpdf
\else
\fi

\begin{document}
%
\title{Low-rank Matrix Factorization under General Mixture Noise Distributions}
%
%
%

\author{Xiangyong Cao,
    Qian~Zhao,
	Deyu~Meng$^\ast$,~\IEEEmembership{Member,~IEEE,}
	Yang~Chen,~Zongben~Xu
\thanks{Xiangyong Cao, Qian Zhao, Deyu Meng, Yang Chen, and Zongben Xu are with the School of Mathematics and Statistics and Ministry of Education Key Lab of Intelligent Networks and Network Security, Xi'an Jiaotong University, Xi'an 710049, China (caoxiangyong45@gmail.com, timmy.zhaoqian@gmail.com, dymeng@mail.xjtu.edu.cn, chengyang9103@gmail.com, zbxu@mail.xjtu.edu.cn)}.
\thanks{$^\ast$Deyu Meng is the corresponding author}.
}
\maketitle

\begin{abstract}
Many computer vision problems can be posed as learning a low-dimensional subspace from high dimensional data. The low rank matrix factorization (LRMF) represents a commonly utilized subspace learning strategy. Most of the current LRMF techniques are constructed on the optimization problems using $L_1$-norm and $L_2$-norm losses, which mainly deal with Laplacian and Gaussian noises, respectively. To make LRMF capable of adapting more complex noise, this paper proposes a new LRMF model by assuming noise as Mixture of Exponential Power (MoEP) distributions and proposes a penalized MoEP (PMoEP) model by combining the penalized likelihood method with MoEP distributions. Such setting facilitates the learned LRMF model capable of automatically fitting the real noise through MoEP distributions. Each component in this mixture is adapted from a series of preliminary  super- or sub-Gaussian candidates. Moreover, by facilitating the local continuity of noise components, we embed Markov random field into the PMoEP model and further propose the advanced PMoEP-MRF model. An Expectation Maximization (EM) algorithm and a variational EM (VEM) algorithm are also designed to infer the parameters involved in the proposed PMoEP and the PMoEP-MRF model, respectively. The superseniority of our methods is demonstrated by extensive experiments on synthetic data, face modeling, hyperspectral image restoration and background subtraction.
\end{abstract}

\IEEEpeerreviewmaketitle
\section{Introduction}
Many computer vision, machine learning, data mining and statistical problems can be formulated as the problem of extracting the intrinsic low dimensional subspace from input high-dimensional data. The extracted subspace tends to deliver the refined latent knowledge underlying data and thus has a wide range of applications including
structure from motion~\cite{tomasi1992shape}, face recognition~\cite{wright2009robust}, collaborative filtering~\cite{koren2008factorization}, information retrieval~\cite{deerwester1990indexing}, social networks~\cite{cheng2012fused}, object recognition~\cite{turk1991eigenfaces}, layer extraction~\cite{ke2001subspace} and plane-based pose estimation~\cite{sturm2000algorithms}.
\begin{figure}
	\centering
	\includegraphics[width=1\linewidth]{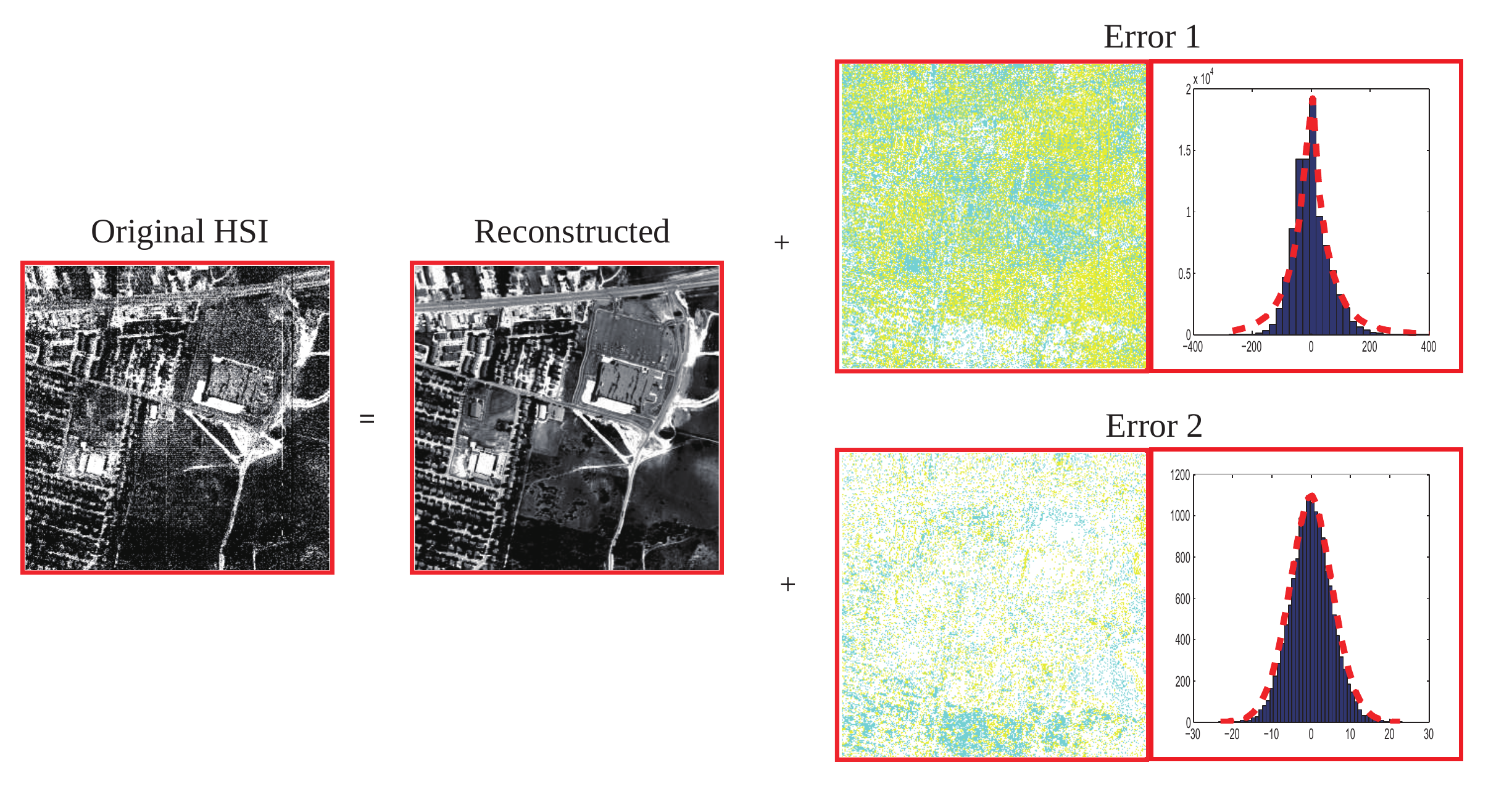}
	\caption{From left to right: Original hyperspectral image (HSI), reconstructed image, two extracted noise images with their histograms by the proposed methods. (Top: $EP_{0.2}$ noise image and histogram. Bottom: $EP_{1.8}$ noise image and histogram).}\label{intro_fig}
	\vspace{-0mm}
\end{figure}

Low rank matrix factorization (LRMF) is one of the most commonly utilized techniques for subspace learning. Given a data matrix $\mathbf{Y}\in\mathcal{R}^{m\times n}$ with entries $y_{ij}s$, the LRMF problem can be mathematically formulated as
\begin{eqnarray}\label{LRMF}
\min_{\mathbf{U},\mathbf{V}}||\mathbf{W}\odot(\mathbf{Y}-\mathbf{U}\mathbf{V}^{T})||,
\end{eqnarray}
where $\mathbf{W}$ is the indicator matrix with $w_{ij} = 0$ if $y_{ij}$ is missing and 1 otherwise, and $\mathbf{U}\in \mathcal{R}^{m\times r}$ and $\mathbf{V}\in \mathcal{R}^{n\times r}$ are low-rank matrices  ($r<\min(m,n)$). The operator $\odot$ denotes the Hadamard product (the component-wise multiplication) and $||\cdot||$ corresponds to a certain noise measure.

Under the assumption of Gaussian noise, it is natural to utilize the $L_2$-norm (Frobenius norm) as the noise measure, which has been extensively studied in LRMF literatures~\cite{srebro2003weighted,buchanan2005damped,okatani2007wiberg, aguiar2008spectrally, zhao2010successively, okatani2011efficient, wen2012solving, mitra2010large}. However, it has been recognized in many real applications that these methods constructed on $L_2$ norm are sensitive to outliers and non-Gaussian noise. In order to introduce robustness, the $L_1$-norm based models~\cite{ke2005robust, eriksson2010efficient,zheng2012practical,kwak2008principal,shu2014robust, ji2010robust} have attracted much attention recently. However, the $L_1$-norm is only optimal for Laplace-like noise and still very limited for handling various types of noise encountered in real problems. Taking the hyper-spectral image (HSI) as an example, it has been investigated in~\cite{zhang2014hyperspectral} that there are mainly two kinds of noise embedded in such type of data, i.e., sparse noise (stripe and deadline) and Gaussian-like noise, as depicted in Fig.~\ref{intro_fig}. The stripe noise is produced by the non-uniform sensor response which conducts the deviation of gray values of the original image continuously towards one direction. This noise always very sparsely located on edges and in texture areas of an image. The deadline noise, which is induced by some damaged sensor, results in zero or very small pixel values of entire columns of images along some HSI bands. The Gaussian-like noise is induced by some random disturbation during the transmission process of hyper-spectral signals. It is easy to see that such kind of complex noise cannot be well fit by either Laplace or Gaussian, which means that neither $L_1$-norm nor $L_2$-norm LRMF models are proper for this type of data.

Very recently, some novel models were presented to expand the availability of LRMF under more complex noise. The key idea is assuming that the noise follows a more complicated mixture of Gaussians (MoG)~\cite{meng2013robust}, which is expected to better fit real noise, since the MoG constructs a universal approximator to any continuous density function in theory~\cite{maz1996approximate}. However, this method still cannot finely adapt real data noise. On one hand, MoG can approximate a complex distribution, e.g. Laplace, only under the assumption that the number of components goes to infinity, while in applications only a finite number of components can be specified. On the other hand, it also lacks a theoretically sound manner to properly select the number of Gaussian mixture components based on the practical noise extent mixed in data. Thus, it is crucial to construct a better strategy with more adaptive distribution modeling capability on data noises beyond MoG.

In this paper, we propose a new LRMF method with a more general noise model to address the aforementioned issues. Specifically, we encode the noise as a mixture distribution of a series of sub- and super-Gaussians (i.e., general exponential power (EP) distribution), and formulate LRMF as a penalized MLE model, called PMoEP model \cite{cao2015PMoEP}. Moreover, by facilitating the local continuity of noise components, we embed Markov random field into the PMoEP model and propose the PMoEP-MRF model. Then we design an Expectation Maximization (EM) algorithm and a variational EM (VEM) algorithm to estimate the parameters involved in the proposed PMoEP model and PMoEP-MRF model, respectively, and prove their convergence. The two new methods are not only capable of adaptively fitting complex real noise by EP noise components with proper parameters, but also able to automatically learn the proper number of noise components from data, and thus can better recover the true low-rank matrix from corrupted data as verified by extensive experiments.

The rest of the paper is organized as follows. In Section II, the related work regarding LRMF is discussed. In Section III, we first present the PMoEP model and the corresponding EM algorithm, and then conduct the convergence analysis of the proposed algorithm. The PMoEP-MRF model and the corresponding variational EM algorithm are proposed in Section IV. In Section V, extensive experiments are conducted to substantiate the superiority of the proposed models over previous methods. Finally, conclusions are drawn in Section VI. Throughout the paper, we denote scalars, vectors, and matrices as the non-bold letters, bold lower case letters, and bold upper case letters, respectively.

\section{Related work}
The $L_2$ norm LRMF with missing data has been studied for decades. Gabriel and Zamir~\cite{gabriel1979lower} proposed a weighted SVD method as the early attempt for this task. They used alternated minimization to find the principal subspace underlying the data. Srebro and Jaakkola~\cite{srebro2003weighted} proposed the Weighted Low-rank Approximation (WLRA) algorithm to enhance efficiency of LRMF calculation. Buchanan and Fitzgibbon~\cite{buchanan2005damped} further proposed a regularized model that adds a regularization term and then adopts the damped newton algorithm to estimate the subspaces. However, it cannot handle large-scale problems due to the infeasibility of computing the Hessian matrix over a large number of variables. Okatani and Deguchi~\cite{okatani2007wiberg} showed that a Wiberg marginalization strategy on $\mathbf{U}$ and $\mathbf{V}$ can provide a better and robust initialization and proposed the Wiberg algorithm that updates $\mathbf{U}$ via least squares while updates $\mathbf{V}$ by a Gauss-Newton step in each iteration. Later, the Wiberg algorithm was extended to a damped version to achieve better convergence by Okatani et al.~\cite{okatani2011efficient}. Aguiar et al.~\cite{aguiar2008spectrally} deduced a globally optimal solution to $L_2$-LRMF with missing data under the assumption that the missing data has a special Young diagram structure. Zhao and Zhang~\cite{zhao2010successively} formulated the $L_2$- norm LRMF as a constrained model to improve its stability in real applications.
Wen et al.~\cite{wen2012solving} adopted the alternating strategy to solve the $L_2$-norm LRMF problem. Mitra et al.~\cite{mitra2010large} proposed an augmented Lagrangian method to solve the $L_2$-norm LRMF problem for higher accuracy. However, all of these methods minimize the $L_2$-norm or its variations and is only optimal for Gaussian-like noise.

To make subspace learning method less sensitive to outliers, some robust loss functions have been investigated. For example, De la Torre and Black~\cite{de2003framework} adopted the Geman-McClure function and then used the iterative reweighted least square (IRLS) method to solve the induced optimization problem. In the last decade, the $L_1$-norm has become the most popular robust loss function along this research line. Ke and Kanade~\cite{ke2005robust} initially replaced the $L_2$-norm with the $L_1$-norm for LRMF, and then solved the optimization by alternated convex programming (ACP) method. Kwak~\cite{kwak2008principal} later proposed to maximize the $L_1$-norm of the projection of data points onto the unknown principal directions instead of minimizing the residue. Eriksson and Hengel~\cite{eriksson2010efficient}
experimentally showed that the ACP approach does not converge to the desired point with high probability, and thus introduced the $L_1$-Wiberg approach to address this issue. Zheng et al.~\cite{zheng2012practical} added more constraints to the factors $\mathbf{U}$ and $\mathbf{V}$ for $L_1$-norm LRMF, and solved the optimization by
ALM, which improved the performance in structure from motion application. Within the probabilistic framework, Wang et al.~\cite{wang2012probabilistic} proposed probabilistic robust matrix factorization (PRMF) that modeled the noise as a Laplace distribution, which has been later extended to fully Bayesian settings by Wang and Yeung~\cite{wang2013bayesian}. However, these methods optimize the $L_1$-norm and thus are only optimal for Laplace-like noise.

Beyond Gaussian or Laplace, other types of noise assumptions have also been attempted recently to make the model adaptable to more complex noise scenarios. Lakshminarayanan et al.~\cite{lakshminarayanan2011robust} assumed that the noise is drawn from a student-t distribution. Babacan et al.~\cite{babacan2012sparse} proposed a Bayesian methods for low-rank matrix estimation modeling the noise as a combination of sparse and Gaussian. To handle more complex noise, Meng and De la Torre~\cite{meng2013robust} modeled the noise as a MoG distribution for LRMF, and later was extended to the Bayesian framework by Chen et al.~\cite{chen2015bayesian} and to RPCA by Zhao et al.~\cite{zhao2014robust}. Although better than traditional methods, these methods are still very limited in dealing with complex noise in real scenarios.

\section{LRMF with MoEP noise}
In this section, we first present the new LRMF model with MoEP noise, called PMoEP model, and then design an EM algorithm to solve it. Finally, we give the convergence analysis of the proposed EM algorithm and the implementation issues.

\subsection{PMoEP model}
In LRMF, from a generative perspective, each element $y_{ij}(i=1,2,\dots,m,j=1,2,\dots,n)$ of the data matrix $\mathbf{Y}$ can be modeled as
\begin{eqnarray}
	y_{ij}=\mathbf{u}_{i}\mathbf{v}_{j}^{T} + e_{ij},
\end{eqnarray}
where $\mathbf{u}_{i}$ and $\mathbf{v}_{i}$ represent the $i^{th}$ row vectors of $\mathbf{U}$ and $\mathbf{V}$, respectively, and $e_{ij}$ is the noise embedded in $y_{ij}$. Instead of assuming that the noise obeys Gaussian~\cite{srebro2003weighted}, Laplace~\cite{ke2005robust} or MoG~\cite{meng2013robust} distributions as previous methods, we assume that the noise $e_{ij}$ follows more flexible mixture of Exponential Power (EP) distributions:
\begin{eqnarray}\label{MoEP}
	\mathbb{P}(e_{ij}) = \sum_{k=1}^K \pi_{k} f_{p_{k}}(e_{ij};0,\eta_k),
\end{eqnarray}
where $\pi_{k}$ is the mixing proportion with $\pi_{k}\geq 0$ and $\sum_{k=1}^{K}\pi_{k}=1$, $K$ is the number of the mixture components and $f_{p_{k}}(e_{ij};0,\eta_k)$ denotes the $k^{th}$ EP distribution with parameter $\eta_{k}$ and $p_{k} (p_{k}>0)$. Let $\mathbf{p}=[p_1,p_2,\dots,p_{K}]$, in which each $p_{k}$ can be variously specified. As defined in~\cite{mineo2005software}, the density function of the EP distribution ($p>0$) with zero mean is
\begin{eqnarray}\label{EP_pdf}
	f_{p}(e;0,\eta) &=& \frac{p\eta^{\frac{1}{p}}}{2\Gamma(\frac{1}{p})}\exp\{-\eta|e|^{p}\},
\end{eqnarray}
where $\eta$ is the precision parameter, $p$ is the shape parameter and $\Gamma(\cdot)$ is the Gamma function. By changing the shape parameter $p$, the EP distribution describes both leptokurtic ($0<p<2$) and platykurtic ($p>2$) distributions. In particular, we obtain the Laplace distribution with $p=1$, the Gaussian distribution with $p=2$ and the Uniform distribution with $p\rightarrow \infty$ (see Fig. \ref{EPpdf}). Therefore, all previous cases including $L_2$, $L_1$, MoG and any combinations of them are just special cases of MoEP. By setting $\eta=1/(p\sigma^{p})$, the EP distribution (\ref{EP_pdf}) can be equivalently written as $EP_{p}(e;0,p\sigma^{p})$.

\begin{figure}
	\centering
	\includegraphics[width=0.85\linewidth]{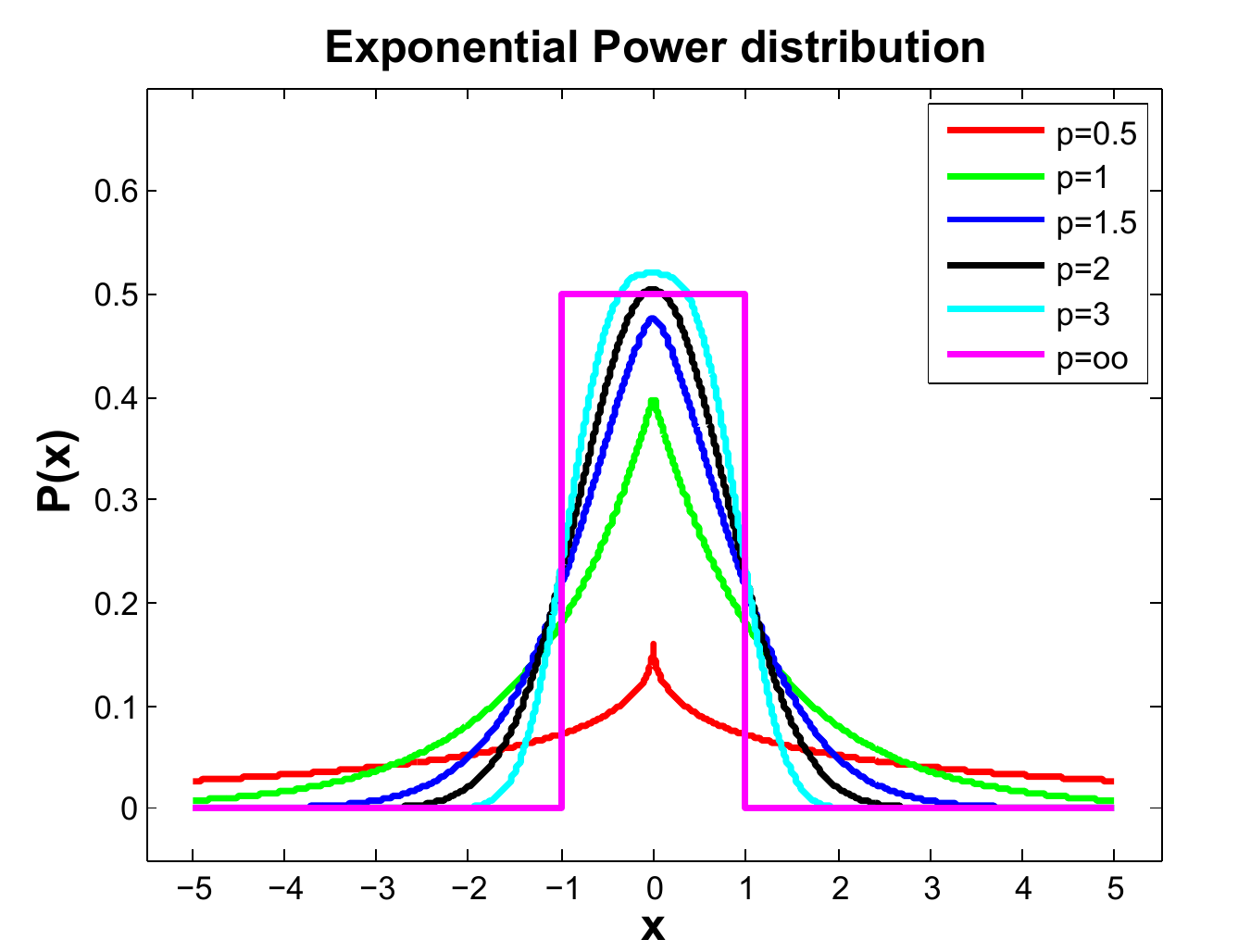}
	\caption{The probability density function of EP distributions.}\label{EPpdf}
\end{figure}

In our model, we assume that each noise $e_{ij}$ is equipped with an indicator variable $\mathbf{z}_{ij}=[z_{ij1},z_{ij2},\dots,z_{ijK}]^{T}$, where $z_{ijk}\in \{0,1\}$ and $\sum_{k=1}^{K}z_{ijk}=1$. $z_{ijk}=1$ implies that the noise $e_{ij}$ is drawn from the $k^{th}$ EP distribution. $\mathbf{z}_{ij}$ obeys a multinomial distribution $\mathbf{z}_{ij}\sim \mathcal{M}(\boldsymbol{\pi})$, where $\boldsymbol{\pi}=[\pi_1,\pi_2,\dots,\pi_{K}]^{T}$. Then we have:
\begin{eqnarray}
	\mathbb{P}(e_{ij}|\mathbf{z}_{ij}) &=& \prod_{k=1}^{K} f_{p_{k}}(e_{ij};0,\eta_{k})^{z_{ijk}},\\
	\mathbb{P}(\mathbf{z}_{ij};\boldsymbol{\pi}) &=& \prod_{k=1}^{K}\pi_{k}^{z_{ijk}}.
\end{eqnarray}
Denoting $\mathbf{E}=(e_{ij})_{m\times n}$, $\mathbf{Z}=(\mathbf{z}_{ij})_{m\times n}$ and $\mathbf{\Theta}=\{\boldsymbol{\pi},\boldsymbol{\eta},\mathbf{U},\mathbf{V}\}$ with $\boldsymbol{\eta}=[\eta_1,\eta_2,\dots,\eta_K]^{T}$, the \textit{complete likelihood function} can then be written as
\begin{eqnarray}
	\mathbb{P}(\mathbf{E},\mathbf{Z};\mathbf{\Theta})&=& \prod_{i,j\in \Omega}\prod_{k=1}^{K}[\pi_{k}f_{p_{k}}(e_{ij};0,\eta_{k})]^{z_{ijk}},
\end{eqnarray}
where $\Omega$ is the index set of the non-missing entries in $\mathbf{Y}$.
Then the \textit{log-likelihood function} is
\begin{equation}\label{loglikelihood}
	l(\mathbf{\Theta})=\log{\mathbb{P}(\mathbf{E};\mathbf{\Theta})}=\log{\sum_{\mathbf{Z}}\mathbb{P}(\mathbf{E},\mathbf{Z};\mathbf{\Theta})},
\end{equation}
and the \textit{complete log-likelihood function} is
\begin{eqnarray}
	l^{C}(\mathbf{\Theta})&=&\log{\mathbb{P}(\mathbf{E},\mathbf{Z};\mathbf{\Theta})}\nonumber \\
	&=&\sum_{i,j\in \Omega}\sum_{k=1}^{K}z_{ijk}[\log{\pi_{k}}+\log{f_{p_{k}}(e_{ij};0,\eta_{k})}].
\end{eqnarray}

As aforementioned in introduction, determining the number of components $K$ is an important problem for the mixture model. Thus, various model selection techniques can be readily employed to resolve this issue. Most conventional methods are based on the likelihood function and some information theoretic criteria, such as AIC and BIC. However, Leroux~\cite{leroux1992consistent} showed that these criteria may overestimate the true number of components. On the other hand, Bayesian approaches~\cite{ormoneit1998averaging,zivkovic2004recursive} have also been used to find a suitable number of components of the finite mixture model. But the computation burden and statistical properties of the Bayesian method limit its use to a certain extent. Here we adopt a recently proposed method by Huang et al.~\cite{huang2013model} for this aim of selecting EP mixture number, and construct the following  penalized MoEP (PMoEP) model:
\begin{eqnarray}\label{PMoEPmodel}
	\max_{\mathbf{\Theta}}\left\{l_{P}^{C}(\mathbf{\Theta})=l^{C}(\mathbf{\Theta})-P(\boldsymbol{\pi};\lambda)\right\},
\end{eqnarray}
where
\begin{eqnarray}
	P(\boldsymbol{\pi};\lambda)=n\lambda \sum_{k=1}^{K}D_{k}\log{\frac{\epsilon+\pi_{k}}{\epsilon}},
\end{eqnarray}
with $\epsilon$ being a very small positive number, $\lambda$ being a tuning parameter ($\lambda>0$), and $D_{k}$ being the number of free parameters for the $k^{th}$ component. In the proposed PMoEP model, $D_{k}$ equals 2 (for $\pi_{k}$ and $\eta_{k}$).

\subsection{EM algorithm for PMoEP model}
In this subsection, we propose an EM algorithm to solve the proposed PMoEP model (\ref{PMoEPmodel}). The EM algorithm is an iterative procedure and thus we assume that $\mathbf{\Theta}^{(t)}=\{\{\boldsymbol{\pi}^{(t)}\},\{\boldsymbol{\eta}^{(t)}\},\mathbf{U}^{(t)},\mathbf{V}^{(t)}\}$ is the estimation at the $t^{th}$ iteration. In the following, we will introduce the two steps of the proposed EM algorithm.

In the E step, we compute the conditional expectation of $z_{ijk}$ given $e_{ij}$ by the Bayes' rule:
\begin{eqnarray}\label{updategamma}
	\gamma_{ijk}^{(t+1)} = \frac{\pi_{k}^{(t)}f_{p_k}(y_{ij}-\mathbf{u}_{i}^{(t)}(\mathbf{v}_{j}^{(t)})^{T})|0,\eta_{k}^{(t)})}
	{\sum_{l=1}^{K}\pi_{l}^{(t)}f_{p_l}(y_{ij}-\mathbf{u}_{i}^{(t)}(\mathbf{v}_{j}^{(t)})^{T})|0,\eta_{l}^{(t)}))}.
\end{eqnarray}
Then, it is easy to construct the so-called $Q$ function:
\vspace{-2pt}
\begin{eqnarray}
	\begin{split}
		Q(\mathbf{\Theta},\mathbf{\Theta}^{(t)})\!&=\!\sum_{i,j\in \Omega,k}\gamma_{ijk}^{(t+1)}[\log{f_{p_k}(y_{ij}\!-\!\mathbf{u}_{i}\mathbf{v}_{j}^{T};\eta_{k})}\!+\!\log{\pi_{k}}]\nonumber\\
		& - n\lambda \sum_{k=1}^{K}D_{k}\log{\frac{\epsilon+\pi_{k}}{\epsilon}}.
	\end{split}
\end{eqnarray}

In the M-step, we update $\mathbf{\Theta}$ by maximizing the $Q$ function.
For $\boldsymbol{\pi}$ and $\boldsymbol{\eta}$, it is easy to obtain the update equations by taking the first derivative of $Q$ with respect to them respectively, and finding the zero points through:
\begin{equation}\label{updatepi}
	\pi_{k}^{(t+1)}\!=\!\max\left\{0,\frac{1}{1\!-\!\lambda \hat{D}}\left[\frac{\sum_{i,j\in \Omega}\gamma_{ijk}^{(t+1)}}{|\Omega|}\!-\!\lambda D_{k}\right]\right\},
\end{equation}
\begin{equation}\label{updatetheta}
	\eta_{k}^{(t+1)}\!=\!\frac{N_{k}}{p_{k}\sum_{i,j\in \Omega}\gamma_{ijk}^{(t+1)}|y_{ij}\!-\!\mathbf{u}_{i}^{(t)}(\mathbf{v}_{j}^{(t)})^{T}|^{p_{k}}},
\end{equation}
where $\hat{D} = \sum_{k=1}^{K}D_{k} = 2K$, $N_{k} = \sum_{i,j\in \Omega}\gamma_{ijk}^{(t+1)}$ and $|\Omega|$ is the number of non-missing elements. To update $\mathbf{U}, \mathbf{V}$, we need to maximize the following function:
\begin{eqnarray}\label{update_s}
	-\sum_{i,j\in \Omega}\sum_{k=1}^{K}\gamma_{ijk}^{(t+1)}
	\eta_{k}^{(t+1)}|y_{ij}-\mathbf{u}_{i}^{(t)}(\mathbf{v}_{j}^{(t)})^{T}|^{p_k},
\end{eqnarray}
which is equivalent to solving\footnote{The $p$-norm of a matrix is defined as $||\mathbf{X}||_{p}=(\sum_{i,j}|x_{ij}|^{p})^{\frac{1}{p}}$.}
\begin{equation}\label{subproblem_uv}
	\min_{\mathbf{U}, \mathbf{V}}\sum_{k=1}^{K}||\mathbf{W}_{(k)}\odot (\mathbf{Y}-\mathbf{U}\mathbf{V}^{T})||_{p_{k}}^{p_{k}},
\end{equation}
where the element $w_{(k)ij}$ of $\mathbf{W}_{(k)}\in \mathcal{R}^{m\times n}(k=1,\dots,K)$ is\\
\begin{displaymath}
	w_{(k)ij}=\begin{cases}(\eta_{k}^{(t+1)}\gamma_{ijk}^{(t+1)})^{\frac{1}{p_{k}}}, \quad i,j\in\Omega\\~~~~~~~~~~0, ~~~~~~~~~~~ i,j\notin\Omega \end{cases}.
\end{displaymath}
To solve (\ref{subproblem_uv}), we resort to augmented Lagrange multipliers (ALM) method. By introducing auxiliary variable $\mathbf{L}=\mathbf{U}\mathbf{V}^{T}$, (\ref{subproblem_uv}) can be equivalently rewritten as
\begin{equation}\label{subproblem_uv2}
	\begin{split}
		&\min_{\mathbf{U},\mathbf{V}}\sum_{k=1}^{K}||\mathbf{W}_{(k)}\odot (\mathbf{Y}-\mathbf{L})||_{p_{k}}^{p_{k}},
		\quad s.t~\mathbf{L} = \mathbf{U}\mathbf{V}^{T}.
	\end{split}
\end{equation}
The augmented Lagrangian function can be written as:
\vspace{-2pt}
\begin{equation}\label{lagrangefunc}
	\begin{split}
		L(\mathbf{U},\mathbf{V},\mathbf{L},\mathbf{Y},\rho)&=\sum_{k=1}^{K}||\mathbf{W}_{(k)}\odot (\mathbf{Y}\!-\!\mathbf{L})||_{p_{k}}^{p_{k}}\\
		&+\langle\mathbf{\Lambda},\mathbf{L}\!-\!\mathbf{U}\mathbf{V}^{T}\rangle+\frac{\rho}{2}||\mathbf{L}\!-\!\mathbf{U}\mathbf{V}^{T}||_{F}^{2},
	\end{split}
\end{equation}
where $\mathbf{\Lambda}\in \mathcal{R}^{m\times n}$ is the Lagrange multiplier and $\rho$ is a positive scalar. Then the optimization (\ref{subproblem_uv2}) can be solved by alternatively updating all involved variables and multipliers as follows
\begin{eqnarray}\label{optProcess}
\begin{cases}
	&\!\left(\mathbf{U}^{(s+1)},\mathbf{V}^{(s+1)}\right)\!=\! \underset{\mathbf{U},\mathbf{V}}{\arg\min} L(\mathbf{U},\mathbf{V},\mathbf{L}^{(s)},\mathbf{\Lambda}^{(s)},\rho^{(s)}),\\
	&\!\mathbf{L}^{(s+1)}\!=\! \underset{\mathbf{L}}{\arg\min} L(\mathbf{U}^{(s+1)},\mathbf{V}^{(s+1)},\mathbf{L},\mathbf{\Lambda}^{(s)},\rho^{(s)}),\\
	&\!\mathbf{\Lambda}^{(s+1)}\!=\!\mathbf{\Lambda}^{(s)} + \rho^{(s)}(\mathbf{L}^{(s+1)}\!-\!\mathbf{U}^{(s+1)}(\mathbf{V}^{(s+1)})^{T}),\label{alg1:eq3}\\
	&\!\rho^{(s+1)} \!=\! \alpha\rho^{(s)}\label{alg1:eq4},
\end{cases}
\end{eqnarray}
where $\alpha$ is a preset constant which is slightly larger than 1, guaranteeing the gradually increasing value for $\rho$ in each iteration. Now we discuss how to solve the subproblems involved in the above procedure.

(1) \textit{Update} $\mathbf{U},\mathbf{V}$. The following subproblem needs to be solved:
\begin{equation}\label{uvupdate}
	\min_{\mathbf{U},\mathbf{V}}||\mathbf{L}^{(s)}+\frac{1}{\rho^{(s)}}\mathbf{\Lambda}^{(s)}-\mathbf{U}\mathbf{V}^{T}||_{F}^{2},
\end{equation}
which can be accurately and efficiently solved by the SVD method.

(2) \textit{Update} $\mathbf{L}$. We need to solve the following problem:
\begin{eqnarray}\label{subproblemL}
	\begin{split}
		&\min_{\mathbf{L}}\sum_{k=1}^{K}||\mathbf{W}_{(k)}\!\odot\! (\mathbf{Y}\!-\!\mathbf{L})||_{p_{k}}^{p_{k}}\!+\!\langle\mathbf{\Lambda}^{(s)},\mathbf{L}\rangle\\
		&+\frac{\rho^{(s)}}{2}||\mathbf{L}-\mathbf{U}^{(s+1)}(\mathbf{V}^{(s+1)})^{T}||_{F}^{2}.
	\end{split}
\end{eqnarray}
This problem seems to be more difficult due to its non-convexity and non-smoothness. However, we can divide it into $mn$ independent scalar optimization problems as follows:
\begin{eqnarray}\label{updateE}
	\begin{cases}
		\begin{split}
			&\min_{l_{ij}}\sum_{k}\eta_{k}\gamma_{ijk}|y_{ij}-l_{ij}|^{p_{k}}
			+\frac{\rho^{(s)}}{2}l_{ij}^{2}\\
			&~~~~~~~~~+((\mathbf{\Lambda}_{ij}^{(s)})-\rho^{(s)}\mathbf{u}_{i}\mathbf{v}_{j}^{T})l_{ij},~~~~~~~~(i,j)\in\Omega\\
			&\underset{l_{ij}}{\min}\frac{\rho^{(s)}}{2}l_{ij}^{2}+ ((\mathbf{\Lambda}^{(s)})_{ij}-\rho^{(s)}\mathbf{u}_{i}\mathbf{v}_{j}^{T})l_{ij}.~~~(i,j)\notin\Omega
		\end{split}
	\end{cases}
\end{eqnarray}
Letting $s_{ij} = y_{ij}-l_{ij}$, (\ref{updateE}) is equivalent to
\begin{eqnarray}\label{sij}
	\begin{cases}
		&\underset{s_{ij}}{\min}\frac{1}{2}(t_{ij}-s_{ij})^{2}+ \frac{1}{\rho^{(s)}}\sum_{l}\eta_{l}\gamma_{ijl}|s_{ij}|^{p_{l}},~(i,j)\in\Omega\\
		&\underset{s_{ij}}{\min}\frac{1}{2}(t_{ij}-s_{ij})^{2},~(i,j)\notin\Omega
	\end{cases}
\end{eqnarray}
where $t_{ij}=-\mathbf{u}_{i}\mathbf{v}_{j}^{T}+y_{ij}+\frac{1}{\rho^{(s)}}(\mathbf{\Lambda}_{ij}^{(s)})$. Then,  for each $(i,j)\in\Omega$, (\ref{sij}) is equivalent to the following subproblem:
\begin{eqnarray}\label{subproblem_e}
	\min_{s_{ij}} \frac{1}{2}(t_{ij}-s_{ij})^{2} + \frac{1}{\rho}\sum_{l=1}^{K}\eta_{l}\gamma_{ijl}|s_{ij}|^{p_{l}}.
\end{eqnarray}
This problem requires to optimize a scalar variable, and we take its first derivative with respect to $s_{ij}$ and then adopt the well-known Newton method to easily approach a local minimum of it. The procedure of updating $\mathbf{L}$ by ALM method can then be listed in Algorithm~\ref{alg1}.
\begin{algorithm}[H]
	\caption{ALM method for solving~(\ref{subproblem_uv}).}
	\label{alg1}
	\begin{algorithmic}[1]
		\REQUIRE The initialization of $\mathbf{L}^{(0)}$, $\mathbf{\Lambda}^{(0)}$ and $s=0$.
		\ENSURE $\mathbf{U}$ and $\mathbf{V}$.
		\WHILE{ not converged }
		\STATE Updating $\mathbf{U}^{(s+1)}$ and  $\mathbf{V}^{(s+1)}$ via Eq.~(\ref{uvupdate});
		\STATE Updating $\mathbf{L}^{(s+1)}$ via Eqs.~(\ref{sij}) and (\ref{subproblem_e}).
		\STATE Updating $\mathbf{\Lambda}^{(s+1)}$ via Eq.~(\ref{alg1:eq3}).
		\STATE Updating $\alpha^{(s+1)}$ via Eq.~(\ref{alg1:eq4}).
		\ENDWHILE
	\end{algorithmic}
\end{algorithm}

{\bf{Remark:}} If $f_{k}$ is specified as the density of a Gaussian distribution, the PMoEP model degenerates to the penalized MoG (PMoG) model. The optimization process of the PMoG model is almost the same as the PMoEP except the minimization form of (\ref{subproblem_uv}). In this case, the optimization problem (18) has the following form
\begin{eqnarray}\label{PMoG_dif}
	\min_{\mathbf{U}, \mathbf{V}}||\mathbf{\tilde{W}}\odot (\mathbf{Y}-\mathbf{U}\mathbf{V}^{T})||_{2}^{2},
\end{eqnarray}
and then any off-the-shelf weighted $L_2$ norm LRMF method can be adopted to solve it. It should be noted that the PMoG method so conducted is different from the previous MoG method~\cite{meng2013robust} due to its augmented automatic mixture-component-number learning capability.

The proposed EM algorithm for PMoEP model can now be summarized in Algorithm~\ref{alg2}.

\begin{algorithm}[H]
	\caption{EM Algorithm for PMoEP LRMF.} \label{alg2}
	\begin{algorithmic}[1]
		\REQUIRE Data $\mathbf{Y}$; The algorithm parameters: rank $r$ and $\lambda$.
		\renewcommand{\algorithmicrequire}{\textbf{Initialization:}}
		\ENSURE Parameter $\mathbf{\Theta}$, the number of mixture components $K_{final}$ and posterior probability $\boldsymbol{\gamma}=(\gamma_{ijk})_{m\times n\times K_{final}}$.
		\REQUIRE $\mathbf{\Theta}^{(t)}\!=\!\{\boldsymbol{\pi}^{(t)},\boldsymbol{\eta}^{(t)}, \mathbf{U}^{(t)}, \mathbf{V}^{(t)}\}$, the number of initial mixture components $K_{start}$, preset candidates $\mathbf{p}=[p_1,\dots,p_{K_{start}}]$, tolerance $\epsilon$ and $t=0$.
		\WHILE { not converged }
		\STATE Updating $\boldsymbol{\gamma}^{(t)}$ via Eq.~\eqref{updategamma};\\
		\STATE Updating $\boldsymbol{\pi}^{(t)}$ via Eq.~\eqref{updatepi}, and removing the component with $\pi_{k}^{(t)}=0$;\\
		\STATE Updating $\boldsymbol{\eta}^{(t)}$ via Eq.~\eqref{updatetheta};\\
		\STATE Updating $\mathbf{U}^{(t)}, \mathbf{V}^{(t)}$ via Algorithm~\ref{alg1}.\\
		\STATE $t = t + 1;$
		\ENDWHILE
	\end{algorithmic}
\end{algorithm}

\subsection{Convergence Analysis of EM algorithm}
In this subsection, we show the convergence property of the proposed EM algorithm for PMoEP model.
\begin{theorem}\label{theorem1}
	Let $l_{P}^{C}(\mathbf{\Theta}) = l(\mathbf{\Theta})-P(\boldsymbol{\pi};\lambda)$, where $l(\Theta)$ is defined in (\ref{loglikelihood}). If we assume that $\{\mathbf{\Theta}^{(t)}\}$ is the sequence generated by Algorithm \ref{alg2} and the sequence of likelihood values $\{l_{P}^{C}(\mathbf{\Theta}^{(t)})\}$ is bounded above, then there exits a constant $l^{\star}$ such that
	\begin{equation}
		\lim_{t\rightarrow \infty}l_{P}^{C}(\mathbf{\Theta}^{(t)}) = l^{\star},
	\end{equation}
	where
	\begin{equation}
		 \mathbf{\Theta}^{(t)}\!=\!\underset{\mathbf{\Theta}}{\arg\max}\left\{\Omega(\mathbf{\Theta}|\mathbf{\Theta}^{(t-1)})\!+\!P(\boldsymbol{\pi}^{(t-1)};\lambda)\!-\!P(\boldsymbol{\pi};\lambda)\right\},
	\end{equation}
	and
	\begin{equation}
		\Omega(\mathbf{\Theta}|\mathbf{\Theta}^{(t-1)})\!=\!\sum_{\mathbf{Z}}\mathbb{P}(\mathbf{Z}|\mathbf{E};\mathbf{\Theta}^{(t-1)})
		\log{\frac{\mathbb{P}(\mathbf{E},\mathbf{Z};\mathbf{\Theta})}{\mathbb{P}(\mathbf{E},\mathbf{Z};\mathbf{\Theta}^{(t-1)})}}.
	\end{equation}
\end{theorem}
The proof is listed in Appendix A.

\subsection{Implementation Issues}
In the proposed PMoEP algorithm, there are three involved preset parameters, $K_{start}$, $p$ and $\lambda$. Throughout all our experiments, we just simply set $K_{start}$ as a not large number as $4-10$ based on a coarse empirical estimate on the noise complexity inside data. Once $K_{start}$ is initialized, the length of vector $\mathbf{p}=[p_{1},p_{2},\dots,p_{K_{start}}]$ in PMoEP is determined. In all our experiments, the elements in $\mathbf{p}$ are selected ranging over the interval between 0.1 and 2. For the setting of parameter $\lambda$ , we first provide a series of candidates $\lambda$ and then adopt the modified BIC to select a good $\lambda$ among these candidates based on the modified BIC criterion. This criterion has been proven to be able to yield consistent component number estimation of the finite Gaussian mixture model~\cite{huang2013model}. Specifically, the modified BIC criterion is defined as
\begin{equation}\label{bic}
\mbox{BIC}(\lambda)\!=\! \sum_{i,j\in\Omega}\log{\{\sum_{k=1}^{\hat{K}}\hat{\pi}_{k}f_{k}(e_{ij};\hat{\eta}_{k})\}}\!-\! \frac{1}{2}(\sum_{k=1}^{\hat{K}}D_{k})\log{|\Omega|}.
\end{equation}
Then we can select the proper $\hat{\lambda}$ by
\begin{equation}
\hat{\lambda} = \arg\max_{\lambda}\mbox{BIC}(\lambda),
\end{equation}
where $|\Omega|$ is the number of non-missing elements, $\hat{K}$ is the estimate of the number of components, $\hat{\pi}_{k}$ is the estimate of parameter $\pi_{k}$, and $\hat{\eta}_{k}$ is the estimate of parameter $\eta_{k}$ for maximizing (\ref{PMoEPmodel}) for a given $\lambda$.

\section{PMoEP with Markov Random Field}
In this section, we first propose an advanced PMoEP-MRF model. Then, we introduce a variational EM (VEM) algorithm to solve it. Finally, we also show the convergence analysis for the proposed algorithm.
\subsection{PMoEP-MRF Model}
In some practical applications, we often have certain noise prior knowledge. By introducing the prior into modeling, noise can be more appropriately modeled and thus the performance of the model is expected to be further improved. In video data, we can utilize the spatial and temporal smoothness prior. Specifically, for a certain pixel in one video frame, the pixels located near it both spatially and temporally tend to have similar distribution to it. Therefore, by facilitating the local continuity of noise components, we can embed Markov Random Field (MRF) into the PMoEP model. Note that the random variable $\mathbf{z}_{ij}$ determines the cluster label of noise $e_{ij}$ in PMoEP model, and the aforementioned spatial and temporal relationships among adjacent pixels imply that they incline to possess similar $\mathbf{z}_{ij}$ values. Therefore, we integrate into the distribution of  $\mathbf{z}_{ij}$ with such prior smoothness knowledge as:
\begin{equation}
\mathbf{z}_{ij}\sim \mathcal{M}(\mathbf{z}_{ij};\boldsymbol{\pi})\prod_{(p,q)\in\mathcal{N}(i,j)}\psi(\mathbf{z}_{ij},\mathbf{z}_{pq}),
\end{equation}
where
\begin{equation}
\psi(\mathbf{z}_{ij},\mathbf{z}_{pq})=
\frac{1}{C}\prod_{k}\exp\left[\tau(2z_{ijk}\!-\!1)(2z_{pqk}\!-\!1)\right],
\end{equation}
where $\tau$ is a positive scalar parameter (we set $\tau=10$ in experiments), $C$ is a normalization constant of $\psi(\mathbf{z}_{ij},\mathbf{z}_{pq})$ and $\mathcal{N}(i,j)$ is the neighborhood of the $(i,j)$ entry.
Specifically, when $z_{ijk}$ and $z_{pqk}$ achieve the same value (0 or 1), $\psi(\mathbf{z}_{ij},\mathbf{z}_{pq})$ will have higher value, and thus this term readily encode the expected prior information. After defining the new distribution of $\mathbf{z}_{ij}$, the  distribution of $\mathbf{Z}$ can be written as
\begin{eqnarray}
\begin{split}
\mathbb{P}(\mathbf{Z};\boldsymbol{\pi})&=\frac{1}{C}\prod_{i,j\in\Omega,k}\pi_{k}^{z_{ijk}}\\
&\prod_{i,j\in\Omega,k}\prod_{(p,q)\in \mathcal{N}(i,j)}\exp\left[\tau(2z_{ijk}\!-\!1)(2z_{pqk}\!-\!1)\right].
\end{split}
\end{eqnarray}
Then, the \textit{complete likelihood function} can be written as
\begin{eqnarray}
\begin{split}
\mathbb{P}(\mathbf{E},\mathbf{Z};\mathbf{\Theta})&=\mathbb{P}(\mathbf{E}|\mathbf{Z};\boldsymbol{\eta})\mathbb{P}(\mathbf{Z};\boldsymbol{\pi})\\
&= \frac{1}{C}\prod_{i,j\in \Omega,k}[\pi_{k}f_{p_k}(y_{ij}\!-\!\mathbf{u}_{i}\mathbf{v}_{j}^{T};0,\eta_{k})]^{z_{ijk}}\\
&\prod_{i,j\in\Omega,k}\prod_{(p,q)\in \mathcal{N}(i,j)}\!\!\!\!\!\exp\!\left[\tau(2z_{ijk}\!-\!1)(2z_{pqk}\!-\!1)\right],
\end{split}
\end{eqnarray}
and the \textit{complete log-likelihood function} is
\begin{eqnarray}
\begin{split}
l^{C}(\mathbf{\Theta})&\!=\!\log{\mathbb{P}(\mathbf{E},\mathbf{Z};\mathbf{\Theta})}\\
&\!=\!\sum_{i,j\in \Omega,k}z_{ijk}[\log{\pi_{k}}\!+\!\log{f_{p_k}(y_{ij}\!-\!\mathbf{u}_{i}\mathbf{v}_{j}^{T};0,\eta_{k})}]\\
\quad &\!+\!\tau\sum_{i,j\in\Omega,k}\sum_{(p,q)\in \mathcal{N}(i,j)}\!\!\!(2z_{ijk}\!-\!1)(2z_{pqk}\!-\!1) \!+\!const.
\end{split}
\end{eqnarray}
In the next section, we will introduce a variational EM algorithm to solve this PMoEP-MRF model in detail.
\vspace{-5pt}
\subsection{Variational EM algorithm for PMoEP-MRF model}
Since EM requires the computation of conditional distribution $\mathbb{P}(\mathbf{Z}|\mathbf{E})$ which is not tractable. In such PMoEP-MRF model, we resort to the variational method that aims at optimizing a lower bound of $\log{\mathcal{L}(\mathbf{E})}$, denoted by
\begin{eqnarray}
\mathcal{J}(R_{\mathbf{E}}) = \log{\mathcal{L}(\mathbf{E})} - KL[R_{\mathbf{E}}(\mathbf{Z}),\mathbb{P}(\mathbf{Z}|\mathbf{E})],
\end{eqnarray}
where $KL$ denotes the Kullback$-$Leibler divergence, $\mathbb{P}(\mathbf{Z}|\mathbf{E})$ is the true conditional distribution of the indicator variables $\mathbf{Z}$ given $\mathbf{E}$, and $R_{\mathbf{E}}(\mathbf{Z})$ is an approximation of the conditional distribution. $\mathcal{J}(R_{\mathbf{E}})$ equals to $\log{\mathcal{L}(\mathbf{E})}$ if and only if $R_{\mathbf{E}}(\mathbf{Z})=\mathbb{P}(\mathbf{Z}|\mathbf{E})$.

As shown above, we are not able to calculate $\mathbb{P}(\mathbf{Z}|\mathbf{E})$,
so we will look for the best (in terms of $KL$
divergence) $R_{\mathbf{E}}(\mathbf{Z})$ in a certain class of distributions. Specifically, we constrain the variational distribution $R_{\mathbf{E}}(\mathbf{Z})$ to have the following form:
\begin{eqnarray}
R_{\mathbf{E}}(\mathbf{Z})&=&\prod_{i,j}R(\mathbf{z}_{ij};\boldsymbol{\gamma}_{ij}),
\end{eqnarray}
where $R(\mathbf{z}_{ij};\boldsymbol{\gamma}_{ij})=\prod_{ij}\prod_{k}\gamma_{ijk}^{z_{ijk}},\sum_k \gamma_{ijk} = 1$, and $\boldsymbol{\gamma}$ is the variational parameter. Then, the lower bound $\mathcal{J}(R_{\mathbf{E}})$ to be maximized can be written as
\begin{eqnarray}\label{lowbound}
\begin{split}
\mathcal{J}(R_{\mathbf{E}}) &= E_{R_{\mathbf{E}}(\mathbf{Z})}\{\log{\mathbb{P}(\mathbf{E},\mathbf{Z})}\} - E_{R_{\mathbf{E}}(\mathbf{Z})}\{R_{\mathbf{E}}(\mathbf{Z})\},\\
&= \sum_{i,j\in\Omega,k}\left[\log{\pi_{k}+\log{f_{p_k}(e_{ij};0,\eta_{k})}}\right]\\
&+\tau\sum_{i,j\in\Omega,k}\sum_{(p,q)\in \mathcal{N}(i,j)}(2\gamma_{ijk}-1)(2\gamma_{pqk}-1)\\
&-\sum_{i,j\in\Omega,k}\gamma_{ijk}\log{\gamma_{ijk}}+const.
\end{split}
\end{eqnarray}
We can easily adopt alternative search strategy for the maximization problem on $\mathcal{J}(R_{\mathbf{E}})$ by alternatively solving the sub-problems: (i) with respect to $R_{\mathbf{E}}$ and (ii) with respect to parameters $\mathbf{U},\mathbf{V},\boldsymbol{\pi},\boldsymbol{\eta}$. The following Proposition~\ref{pro1} and~\ref{pro2} provide the solutions of optimization problem (i) and (ii), respectively.
\begin{proposition}\label{pro1}
	\textit{(Variational E-step)}~ Given parameters $\mathbf{\Theta}=\{\mathbf{U},\mathbf{V},\boldsymbol{\pi},\boldsymbol{\eta}\}$, the optimal variational parameters $\hat{\gamma}_{ij} = \underset{\boldsymbol{\gamma}}{\arg\max}~\mathcal{J}(R_{\mathbf{E}})$ satisfy the following fixed point relation:
	\begin{equation}\label{updategamma_fixpoint}
	\gamma_{ijk}\propto\pi_{k}f_{p_k}(e_{ij};0,\eta_{k})\exp\{\tau\sum_{(p,q)\in \mathcal{N}(i,j)}\gamma_{pqk}\}.
	\end{equation}	
\end{proposition}
\begin{proof}
Based on (\ref{lowbound}), we maximize $\mathcal{J}(R_{\mathbf{E}})$ with respect to $\boldsymbol{\gamma}_{ij}s$, subject to $\sum_k\gamma_{ijk}=1$, for all $i,j$, i.e. to maximize $\mathcal{J}(R_{\mathbf{E}})+\sum_{ij}[\lambda_{ij}(\sum_{k}\gamma_{ijk}-1)]$ where $\lambda_{ij}$ is the Lagrangian multiplier. The derivative with respect to $\gamma_{ijk}$ is
	\begin{displaymath}
	\log{\pi_{k}\!+\!\log{f_{p_k}(e_{ij};0,\eta_{k})}}\!+\!\tau\!\!\!\!\sum_{(p,q)\in \mathcal{N}(i,j)}\gamma_{pqk}\!-\!\log{\gamma_{ijk}}\!-\!1\!+\!\lambda_{ij}.
	\end{displaymath}
	This derivative is null iff $\gamma_{ijk}$ satisfy the relation given in the
	proposition, and $\exp(-1+\lambda_{ij})$ is the the normalizing constant.
\end{proof}
\vspace{-1pt}
\begin{proposition}\label{pro2}
	\textit{(Variational M-step)}~ Given the variational parameters $\boldsymbol{\gamma}_{ij}s$, the values of parameters $\mathbf{U},\mathbf{V},\boldsymbol{\pi},\boldsymbol{\eta}$ that maximize $\mathcal{J}(R_{\mathbf{E}})$ can be calculated in the same way as the M step in the EM algorithm of PMoEP model.
\end{proposition}

The proposed variational EM algorithm for PMoEP-MRF model can then be summarized in Algorithm~\ref{alg3}.

\begin{algorithm}[H]
	\caption{VEM Algorithm for the PMoEP-MRF Model.} \label{alg3}
	\begin{algorithmic}[1]
		\REQUIRE Data $\mathbf{Y}$, rank $r$, $\tau$ and  $\lambda$.
		\ENSURE Parameter $\mathbf{\Theta}$, mixture components number $K_{final}$ and $\boldsymbol{\gamma}=(\gamma_{ijk})_{m\times n\times K_{final}}$.
		\renewcommand{\algorithmicrequire}{\textbf{Initialization:}}
		\REQUIRE $\mathbf{\Theta}^{(0)}\!=\!\{\boldsymbol{\pi}^{(0)},\boldsymbol{\eta}^{(0)}, \mathbf{U}^{(0)}, \mathbf{V}^{(0)}\}$, the initial mixture components number $K_{start}$,  preset candidates $\mathbf{p}=[p_1,\dots,p_{K_{start}}]$, tolerance $\epsilon$ and $t=0$.
		\WHILE { not converged }
		\STATE Updating $\boldsymbol{\gamma}^{(t)}$ via the fixed-point Eq.~\eqref{updategamma_fixpoint};\\
		\STATE Updating $\boldsymbol{\pi}^{(t)}$ via Eq.~\eqref{updatepi}, and removing the component with $\pi_{k}^{(t)}=0$;\\
		\STATE Updating $\boldsymbol{\eta}^{(t)}$ via Eq.~\eqref{updatetheta};\\
		\STATE Updating $\mathbf{U}^{(t)}, \mathbf{V}^{(t)}$ via Algorithm~\ref{alg1};\\
		\STATE $t = t + 1.$
		\ENDWHILE
	\end{algorithmic}
\end{algorithm}

\subsection{Convergence Analysis of Variational EM algorithm}
In this subsection, we show the convergence property of the proposed EM algorithm for PMoEP model.
\begin{theorem}\label{theorem2}
	Given $\lambda$, Algorithm~\ref{alg3} generates a sequence $\{\{\boldsymbol{\gamma}_{ij}^{(t)}\}, \mathbf{\Theta}^{(t)}\}\}_{t=1}^{\infty}$ which increases $\mathcal{J}(R_{\mathbf{E}})$ such that
	\begin{equation}
	\mathcal{J}(R_{\mathbf{E}};\{\boldsymbol{\gamma}_{ij}^{(t+1)}\}, \mathbf{\Theta}^{(t+1)}\})\geq \mathcal{J}(R_{\mathbf{E}};\{\boldsymbol{\gamma}_{ij}^{(t)}\},\mathbf{\Theta}^{(t)}\}).
	\end{equation}
\end{theorem}
\begin{proof}
	This is a direct consequence of Propositions \ref{pro1} and \ref{pro2}, which both guarantee that $\mathcal{J}(R_{\mathbf{E}})$ monotonically increases in iteration.
\end{proof}

It is easy to see that $\mathcal{J}(R_{\mathbf{E}};\{\boldsymbol{\gamma}_{ij}^{(t)}\},\mathbf{\Theta}^{(t)}\})$ is upper bounded, and thus the convergence of Algorithm 3 can be guaranteed.

\section{Experimental Results}
To evaluate the performance of the proposed PMoEP method, its special case PMoG and the PMoEP-MRF method, we conducted a series of experiments on both synthetic and real data. Five state-of-the-art LRMF methods were considered for comparison, including Mixture of Gaussion method (MoG~\cite{meng2013robust}), Laplace noise methods (CWM~\cite{meng2013cyclic}, RegL1ALM~\cite{zheng2012practical}) and Gaussian noise methods (Damped Wiberg (DW)~\cite{okatani2011efficient} and SVD). All experiments were implemented in Matlab R2014a on a PC with 3.60GHz CPU and 12GB RAM.
\begin{table}
	\newcommand{\tabincell}[2]{\begin{tabular}{@{}#1@{}}#2\end{tabular}}
	\centering
	\caption{\label{table1}The Parameter Selection for PMoG and PMoEP.}
	{\scriptsize
		\scalebox{1}[1]{
			\begin{tabular}{c|c|c}
				\hline
				&\multicolumn{2}{c}{Parameter Selection}\\
				\cline{2-3}
				& PMoG & PMoEP\\
				\hline
				Gaussian & \tabincell{c}{  $K_{final}=1$,\\ $\lambda_{select}=0.01$ } &\tabincell{c}{  $K_{final}=1, p_{select} = 2$,\\$\lambda_{select}=0.15$ }\\
				\hline
				Exponential Power &
				\tabincell{c}{ $K_{final}=3$,\\ $\lambda_{select}=0.001$ } &\tabincell{c}{$K_{final}=1,p_{select}=0.2$,\\ $\lambda_{select}=0.3$}\\
				\hline
				Laplace  &
				\tabincell{c}{  $K_{final}=3$,\\ $\lambda_{select}=0.001$  } &
				\tabincell{c}{$K_{final}=1,p_{select}=1$,\\ $\lambda_{select}=0.1$ }\\
				\hline
				Sparse  &
				\tabincell{c}{  $K_{final}=2$, \\$\lambda_{select}=0.005$ } &
				\tabincell{c}{ $K_{final}=2,p_{select}=[2,2]$\\, $\lambda_{select}=0.005$}\\
				\hline
				Mixuture 1 & \tabincell{c}{ $K_{final}=2$,\\ $\lambda_{select}=0.01$  }  &\tabincell{c}{$K_{final}=2,p_{select}=[1.5,2]$,\\ $\lambda_{select}=0.005$}\\
				\hline
				Mixuture 2 & \tabincell{c}{ $K_{final}=1$,\\ $\lambda_{select}=0.001$  }  &\tabincell{c}{$K_{final}=2,p_{select}=[0.5,2]$,\\ $\lambda_{select}=0.005$}\\
				\hline
			\end{tabular}
		}}
	\end{table}
\subsection{Synthetic simulations}
Several synthetic experiments with different noise settings were designed to compare the performance of the proposed methods and other competing methods. We first randomly generated $30$ low rank matrices with size $40\times20$ and rank 4. Each of these matrices was generated by the multiplication of two low-rank matrices $\mathbf{U}_{gt}\in \mathcal{R}^{40\times 4}$ and $\mathbf{V}_{gt}\in \mathcal{R}^{20\times 4}$, and $\mathbf{Y}_{gt}=\mathbf{U}_{gt}\mathbf{V}_{gt}^{T}$ is the ground truth matrix. Then, we randomly specified 20\% elements of $\mathbf{Y}_{gt}$ as missing entries. Next, we added different types of noise to the non-missing entries as follows: (1) \textit{Gaussian noise}: $\mathcal{N}(0,0.04)$. (2) \textit{Exponential power noise}:\footnote{The method of drawing samples from a general exponential power distribution is introduced in Appendix B.} $EP_{0.2}(0,0.2^{p}p), p=0.2$. (3) \textit{Laplace noise}: $\mathcal{L}(0,0.2)$. (4) \textit{Sparse noise}: 12.5\% of the non-missing entries were corrupted with uniformly distributed noise on [-20,20]. (5) \textit{Mixture noise 1}: 25\% of the entries were corrupted with uniformly distributed noise on [-5,5], 25\% were contaminated with Gaussian noise $\mathcal{N}(0,0.04)$ and the remaining 50\% are corrupted with Gaussian noise $\mathcal{N}(0,0.01)$. (6) \textit{Mixture noise 2}: 37.5\% of the entries were corrupted with $EP(0,0.1^{p}p), p=0.5$, 50\% were contaminated with Laplace noise $\mathcal{L}(0,0.3)$ and the remaining 50\% were corrupted with Gaussian noise $\mathcal{N}(0,0.01)$. Then we get the noisy matrix $\mathbf{Y}_{no}$. Six measures were utilized for performance assessment:
\begin{equation*}
\begin{split}
&C1 \!=\! ||\mathbf{W}\!\odot\!(\mathbf{Y}_{no}\!-\!\tilde{\mathbf{U}}\tilde{\mathbf{V}}^{T})||_{1},\nonumber~C2 \!=\! ||\mathbf{W}\!\odot\! (\mathbf{Y}_{no}\!-\!\tilde{\mathbf{U}}\tilde{\mathbf{V}}^{T})||_{2}, \nonumber\\
&C3 = ||\mathbf{Y}_{gt}-\tilde{\mathbf{U}}\tilde{\mathbf{V}}^{T}||_{1}, \nonumber~~C4 = ||\mathbf{Y}_{gt}-\tilde{\mathbf{U}}\tilde{\mathbf{V}}^{T}||_{2}, \nonumber\\
&C5 = subspace(\mathbf{U}_{gt},\tilde{\mathbf{U}}), \nonumber~~C6 = subspace(\mathbf{V}_{gt},\tilde{\mathbf{V}}),
\end{split}
\end{equation*}
where $\tilde{\mathbf{U}},\tilde{\mathbf{V}}$ are the outputs of the corresponding competing method, and $subspace(\mathbf{U}_1$,$\mathbf{U}_2$) denotes the angle between
subspaces spanned by the columns of $\mathbf{U}_1$ and $\mathbf{U}_2$. Note that $C1$ and $C2$ are the optimization objective function for $L_1$ and $L_2$ norm LRMF problems, while the latter four measures ($C3-C6$) are more faithful to evaluate whether a method recovers the correct subspaces.
	\begin{table}[htp]
		\caption{\label{simuTab1} Performance evaluation on synthetic data. The best results in terms of each criterion are highlighted in bold.}
		\begin{center}
			{\scriptsize
				\scalebox{1}[0.9]{
					\begin{tabular}{ccccccc}
						\toprule[2pt]
						& PMoEP & PMoG & MoG & DW & CWM & RegL1ALM\\
						\midrule[2pt]
						\multicolumn{7}{c}{Gaussian Noise}\\ \hline
						C1 &   40.97         & 41.00   & 41.00  &  41.00  &  39.23  & {\bf{36.60}} \\
						C2 &   {\bf{4.16}}    & {\bf{4.16}}   & {\bf{4.16}}    &  {\bf{4.16}}   &  5.67   & 5.27\\
						C3 &  {\bf{3.27}}     & {\bf{3.27}}     & {\bf{3.27}}    &  {\bf{3.27}}     &  6.01     & 4.94\\
						C4 &  {\bf{3.90e+1}}   & 3.91e+1  & 3.91e+1        &  3.91e+1         &  5.09e+1  & 5.09e+1 \\
						C5 &  {\bf{4.22e-2}}   & {\bf{4.22e-2}}  & {\bf{4.22e-2}} &  {\bf{4.22e-2}}  &  5.71e-2  & 5.33e-2 \\
						C6 &  {\bf{3.01e-2}}  & {\bf{3.01e-2}}  & {\bf{3.01e-2}} &  {\bf{3.01e-2}}  &  4.55e-2  & 3.79e-2\\
						\hline
						\multicolumn{7}{c}{Exponential Power Noise}\\ \hline
						C1 &  3.60e+2          & 3.42e+2  &  3.23e+2  &  4.30e+2  & {\bf{3.21e+2}}  & 3.65e+2\\
						C2 &  1.30e+3          & 1.04e+3  &  1.18e+3  &  {\bf{6.27e+2}}  & 1.17e+3  & 8.51e+2\\
						C3 &  {\bf{1.72e+2}}   & 4.49e+4  &  2.17e+3   &  5.06e+3  & 1.73e+2  & 7.77e+4 \\
						C4 &  {\bf{2.32e+2}}   & 4.67e+2  &  2.60e+2   &  6.29e+2  & 2.40e+2  & 9.68e+2\\
						C5 &  {\bf{3.31e-1}}   & 5.67e-1  &  4.11e-1   &  9.19e-1  & 3.39e-1  & 1.16\\
						C6 &  {\bf{2.19e-1}}   & 4.97e-1  &  2.31e-1   &  8.94e-1  & 2.61e-1  & 1.11\\
						\hline
						\multicolumn{7}{c}{Laplacian Noise}\\ \hline
						C1 &  7.63e+1         & 7.29e+1   &  7.13e+1   &  7.76e+1       & 7.24e+1   & {\bf{6.80e+1}}\\
						C2 &  1.72e+1         & 2.57e+1   &  2.44e+1   &  \bf{1.68e+1}  & 2.16e+1   & 2.10e+1\\
						C3 &  {\bf{1.27e+1}}  & 2.02e+1   &  1.84e+1   &  1.31e+1       & 1.69e+1   & 1.42e+1\\
						C4 &  {\bf{7.54e+1}}  & 9.37e+1   &  8.99e+1   &  7.69e+1       & 8.33e+1   & 7.85e+1\\
						C5 &  {\bf{9.17e-2}}  & 1.15e-1   &  1.07e-1   &  9.22e-2       & 1.07e-1   & 9.80e-2\\
						C6 &  {\bf{6.30e-2}}  & 8.25e-2   &  7.84e-2   &  6.49e-2       & 8.24e-2   & 6.56e-2\\
						\hline
						\multicolumn{7}{c}{Sparse Noise}\\ \hline
						C1 & {\bf{8.12e+2}}  & {\bf{8.12e+2}}  & {\bf{8.12e+2}}  &  1.17e+3   & 8.20e+2   & 8.73e+2\\
						C2 & 1.08e+4  & 1.08e+4  & 1.08e+4  &  {\bf{5.12e+3}}   & 1.06e+4   & 5.95e+3\\
						C3 & {\bf{2.37e-12}} & {\bf{2.37e-12}} & 7.94e-12 &  3.09e+4   & 9.75e+1   & 1.59e+6\\
						C4 & {\bf{2.54e-5}}  & 2.55e-5  & 3.48e-5  &  2.12e+3   & 6.03e+1   & 4.89e+3 \\
						C5 & {\bf{3.87e-8}}  & 3.87e-8  & 6.63e-8  &  1.48      & 2.83e-1   & 1.47\\
						C6 & {\bf{2.28e-8}}  & 2.29e-8  & 4.44e-8  &  1.39      & 6.25e-2   & 1.54\\
						\hline
						\multicolumn{7}{c}{Mixture Noise1}\\ \hline
						C1 &  4.49e+2        & 4.55e+2   & 5.25e+2        &  5.25e+2  & {\bf{4.33e+2}}  & 4.35e+2\\
						C2 &  1.36e+3        & 1.25e+3   & {\bf{8.49e+2}} &  8.51e+2  & 1.12e+3         & 1.16e+3\\
						C3 &  {\bf{1.53e+2}} & 6.52e+4   & 8.98e+2        &  8.93e+2  & 3.01e+2         & 1.56e+4 \\
						C4 &  {\bf{1.66e+2}} & 4.38e+2   & 6.02e+2        &  6.00e+2  & 2.87e+2         & 5.15e+2\\
						C5 &  {\bf{3.28e-1}} & 5.79e-1   & 6.47e-1        &  6.60e-1  & 4.30e-1         & 7.88e-1\\
						C6 &  {\bf{1.18e-1}} & 3.78e-1   & 5.01e-1        &  5.01e-1  & 2.93e-1         & 6.84e-1\\
						\hline
						\multicolumn{7}{c}{Mixture Noise2}\\ \hline
						C1 &  9.01e+1         & 8.93e+1 &  8.76e+1  &  9.60e+1        & 8.83e+1  & {\bf{8.32e+1}}\\
						C2 &  3.37e+1         & 4.04e+1 &  3.99e+1  &  {\bf{2.72e+1}} & 3.53e+1  & 3.42e+1\\
						C3 &  {\bf{1.72e+01}} & 2.62e+1 &  2.49e+1  &  2.13e+1        & 2.40e+1  & 1.87e+1 \\
						C4 &  {\bf{8.57e+01}} & 1.04e+2 &  1.01e+2  &  9.71e+1        & 9.77e+1  & 8.87e+1 \\
						C5 &  {\bf{1.02e-01}} & 1.23e-1 &  1.24e-1  &  1.09e-1        & 1.21e-1  & 1.07e-1\\
						C6 &  {\bf{6.39e-02}} & 8.41e-2 &  8.14e-2  &  7.02e-2        & 8.96e-2  & 6.62e-2\\
						\bottomrule[2pt]
					\end{tabular}}
				}
			\end{center}
		\end{table}	
			\begin{figure}
				\centering
				\includegraphics[width=1\linewidth]{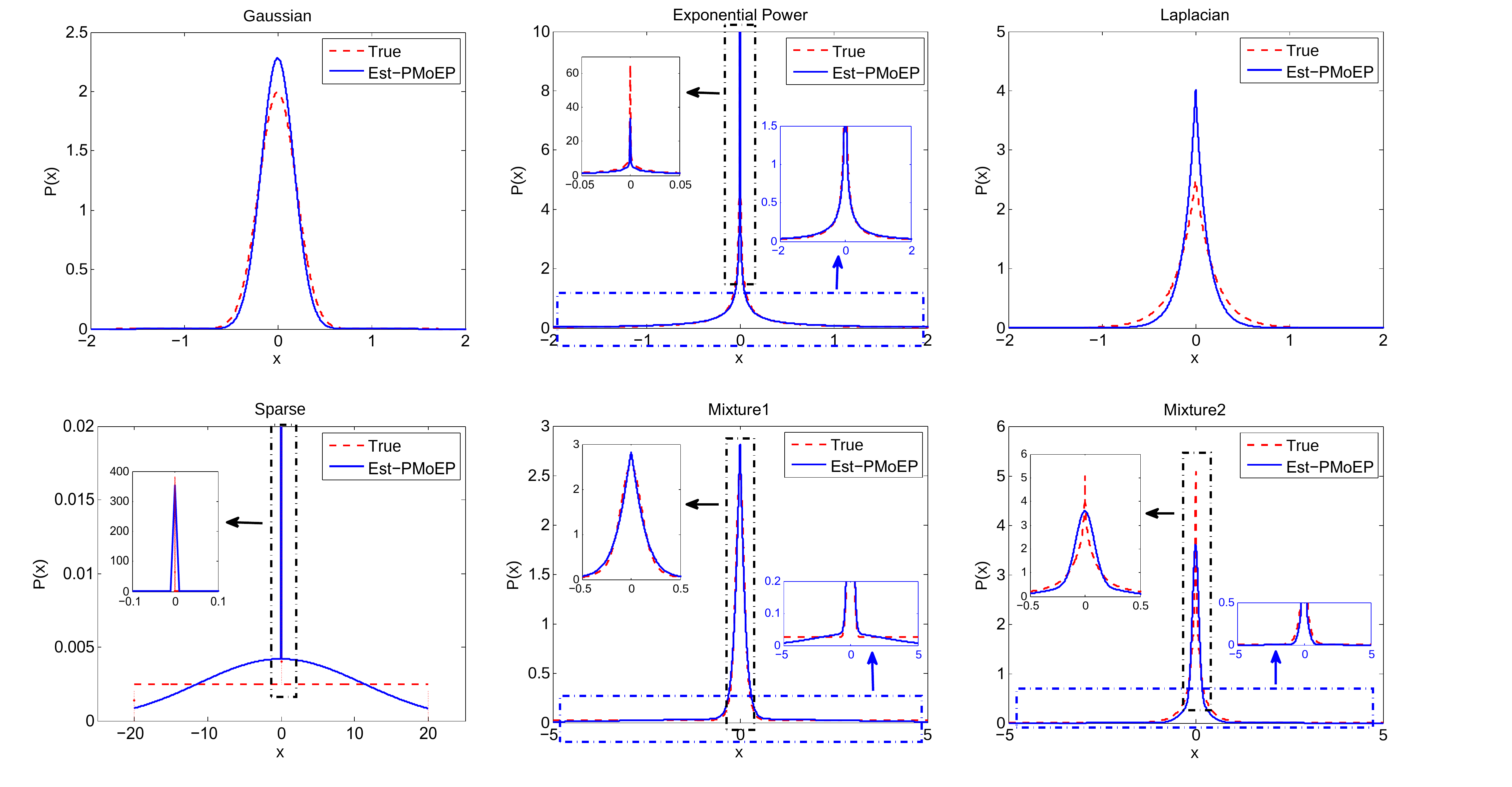}
				\caption{Visual comparison of the ground truth (denote by True) noise probability density functions and those estimated (denote by Est) by the PMoEP method in the synthetic experiments. The embedded sub-figures depict the zoom-in of the indicated portions.}\label{pdfdraw}
			\end{figure}
			
			\begin{figure}
				\centering
				\includegraphics[width=1\linewidth]{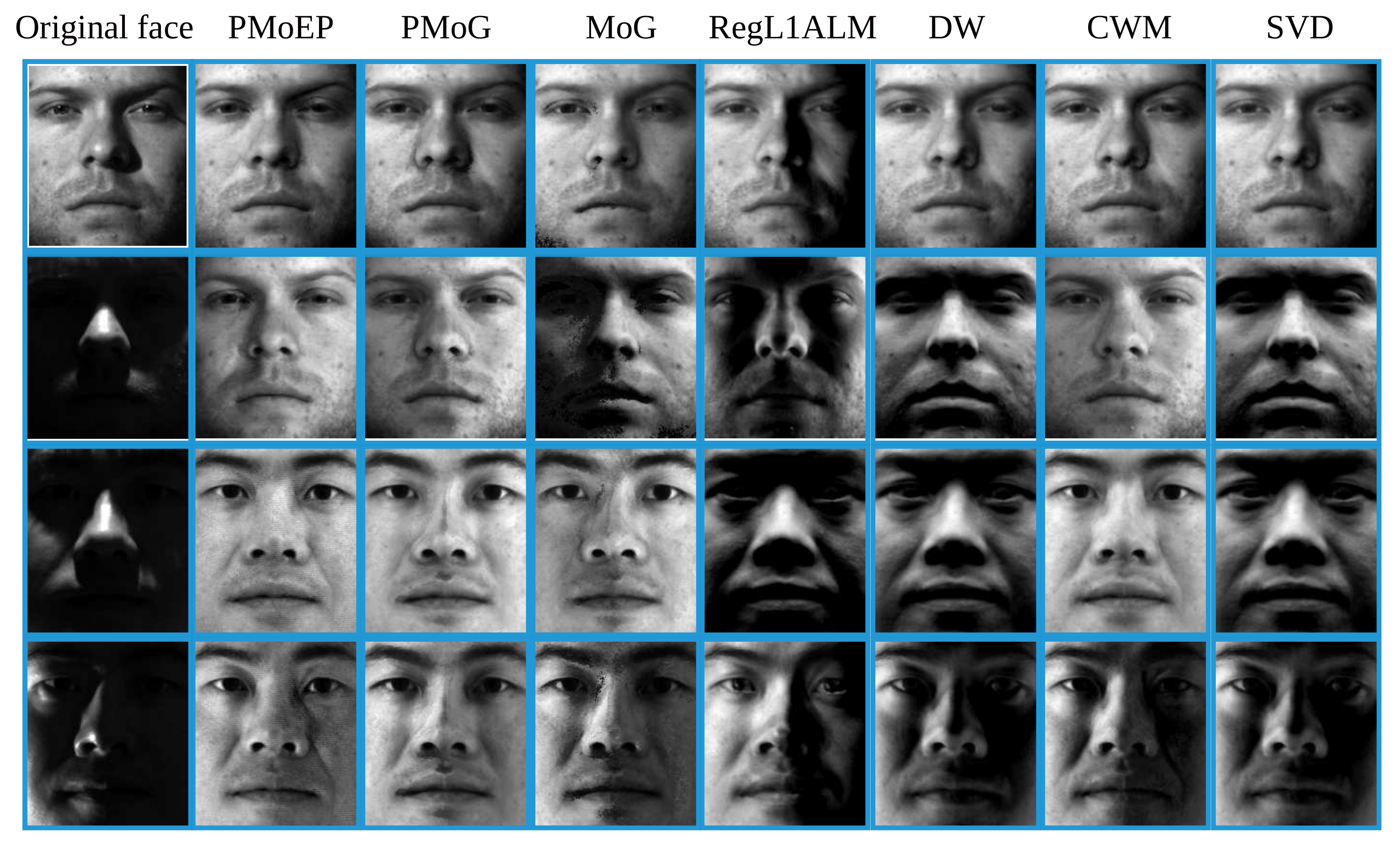}
				\caption{From left to right: original face images, reconstructed faces by PMoEP, PMoG, MoG, RegL1ALM, DW, CWM and SVD.}\label{FaceRecover}
			\end{figure}
We set the rank of all the competing methods to $4$ and adopt the random initialization strategy for all the methods. For each method, we first run with 20 random initializations and then select the best result with respect to the corresponding objective value of the method. The  performance of each method was evaluated as the average results over the 30 random matrices in terms of the six measures, and the results are summarized in Table \ref{simuTab1}. We also report the final selections of the mixture number $K_{final}$ and the corresponding parameter $\lambda_{select}$ for PMoG and PMoEP in Table \ref{table1}.
	
	From Table \ref{simuTab1}, we can observe that $L_2$-norm methods DW, MoG, PMoG and our proposed PMoEP methods achieve the best performance than others in Gaussian noise case. In Laplace noise case, our PMoEP method performs best and $L_1$ method RegL1ALM achieves similar results. When the noise is Exponential Power, PMoEP evidently outperforms other competing methods in term of criteria C3$-$C6. In sparse noise case, PMoEP and PMoG perfom the best and MoG achieves comparable good results with PMoEP. Moreover, when the noise gets more complex, PMoEP achieves the best performance, which attributes to the high flexibility of PMoEP to model unknown complex noise. These results then substantiate that our proposed PMoEP method can estimate a better subspace from the noisy data than other competing methods.
	
    The promising performance of PMoEP method in these cases can be easily explained by Fig. \ref{pdfdraw}, which compares the ground truth noise distributions and the estimated ones by the PMoEP method. It can be easily observed that the estimated noise distributions well match the true ones, which naturally conducts its good reconstruction capability to the true low-rank matrix.

	\subsection{Face modeling}
	This experiment aims to test the effectiveness of PMoG and PMoEP methods in face modeling application. We choose the first and the second subset of the Extended Yale B database\footnote{http://vision.ucsd.edu/~leekc/ExtYaleDatabase/ExtYaleB.html}, and each subset consists of 64 faces of one person with size $192\times168$ and then generate two data matrices, each of which is with size $32256\times 64$. Typical images are shown in the first column of Fig. \ref{FaceRecover}.
	
	We set the rank as $4$~\cite{basri2003lambertian} and adopt two initialization strategies, namely random and SVD for all competing methods. Then we report the best result among the results in terms of the object value of the corresponding model utilized by each method. Some reconstructed faces of different methods are visually compared in Fig. \ref{FaceRecover}.
	
	From Fig. \ref{FaceRecover}, it is easy to observe that,  the proposed PMoEP and PMoG methods, as well as the other competing ones, can remove the cast shadows and saturations in faces. However, our PMoEP and PMoG methods perform better than other ones on faces containing a large dark region. Such face images contain both significant cast shadow and saturation noises, which correspond to the highly dark and bright areas in face, and camera noise~\cite{nakamura2005image}, which is much amplified in the dark areas. Compared with other competing methods, PMoEP method is capable of better extracting such complex noise configurations, and thus leads to its better face reconstruction performance.

\begin{figure}
	\centering
	\includegraphics[width=1.0\linewidth]{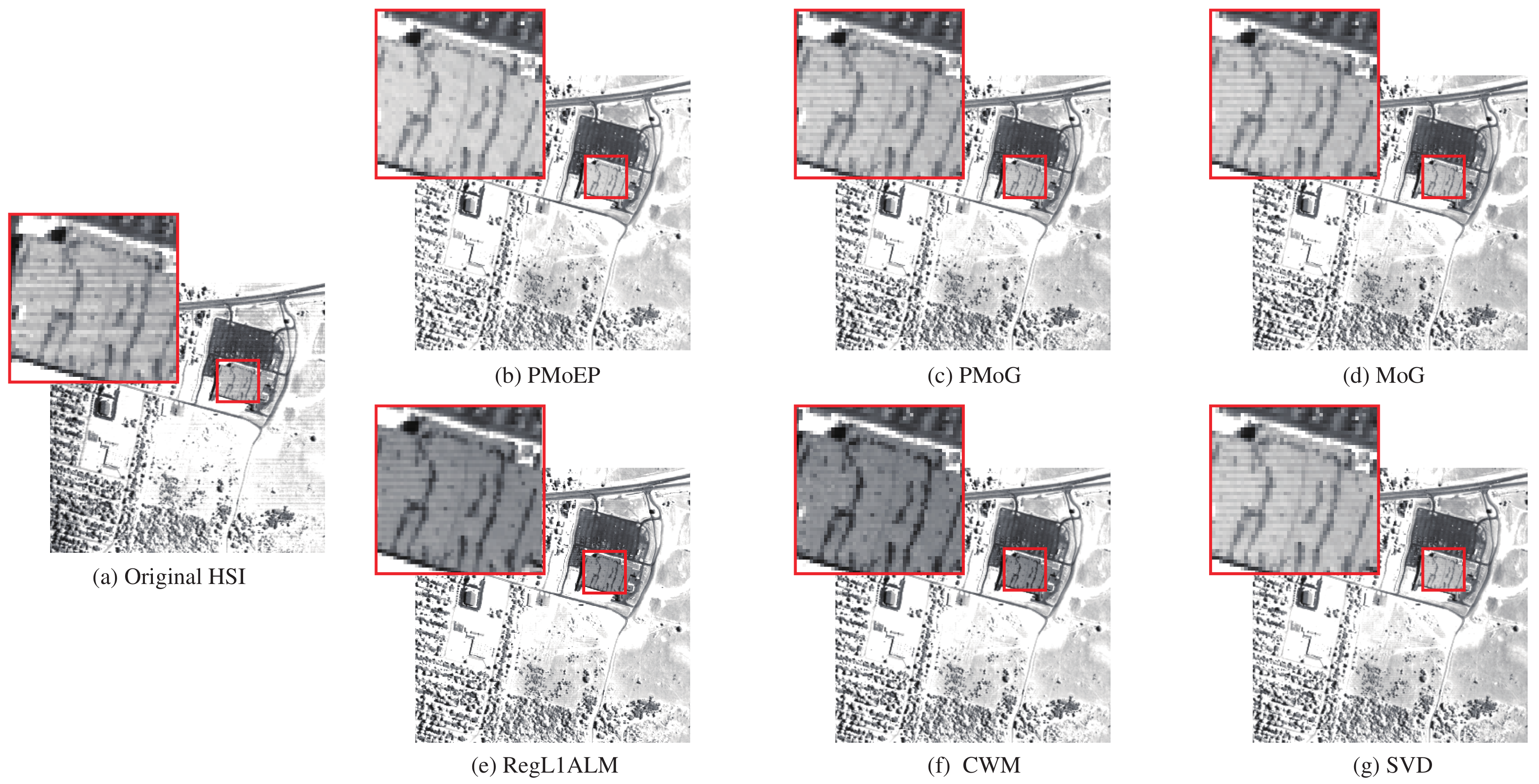}
	\caption{ Restoration results of band 103 in Urban data set: (a) original bands. (b)-(g) reconstructed bands by PMoEP, PMoG, MoG, RegL1ALM, CWM and SVD.}\label{103bandRec}
\end{figure}

\begin{figure}
	\centering
	\includegraphics[width=1.0\linewidth]{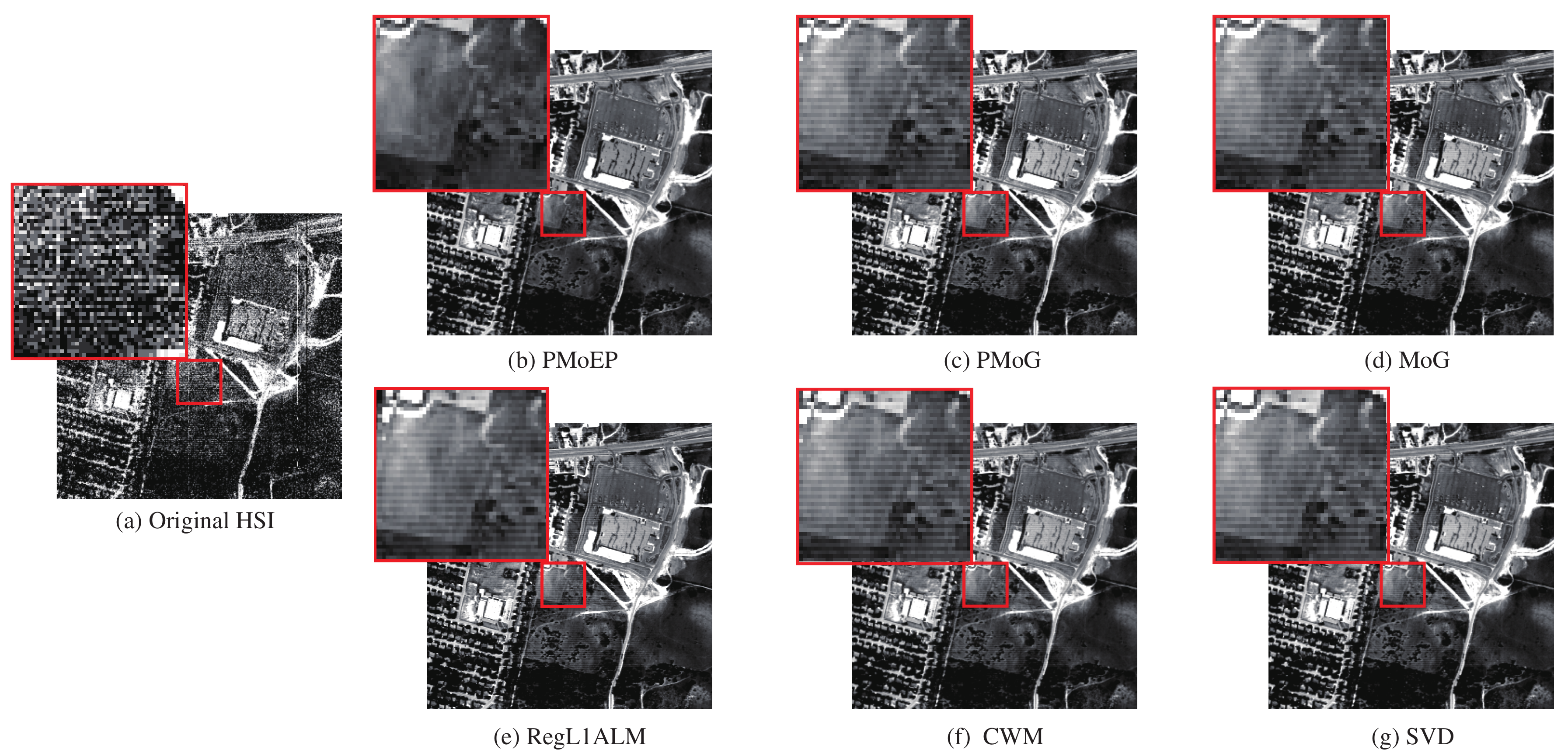}
	\caption{ Restoration results of band 206 in Urban data set: (a) original bands. (b)-(g) reconstructed bands by PMoEP, PMoG, MoG, RegL1ALM, CWM and SVD.}\label{206bandRec}
\end{figure}

\begin{figure}
	\centering
	\includegraphics[width=1.0\linewidth]{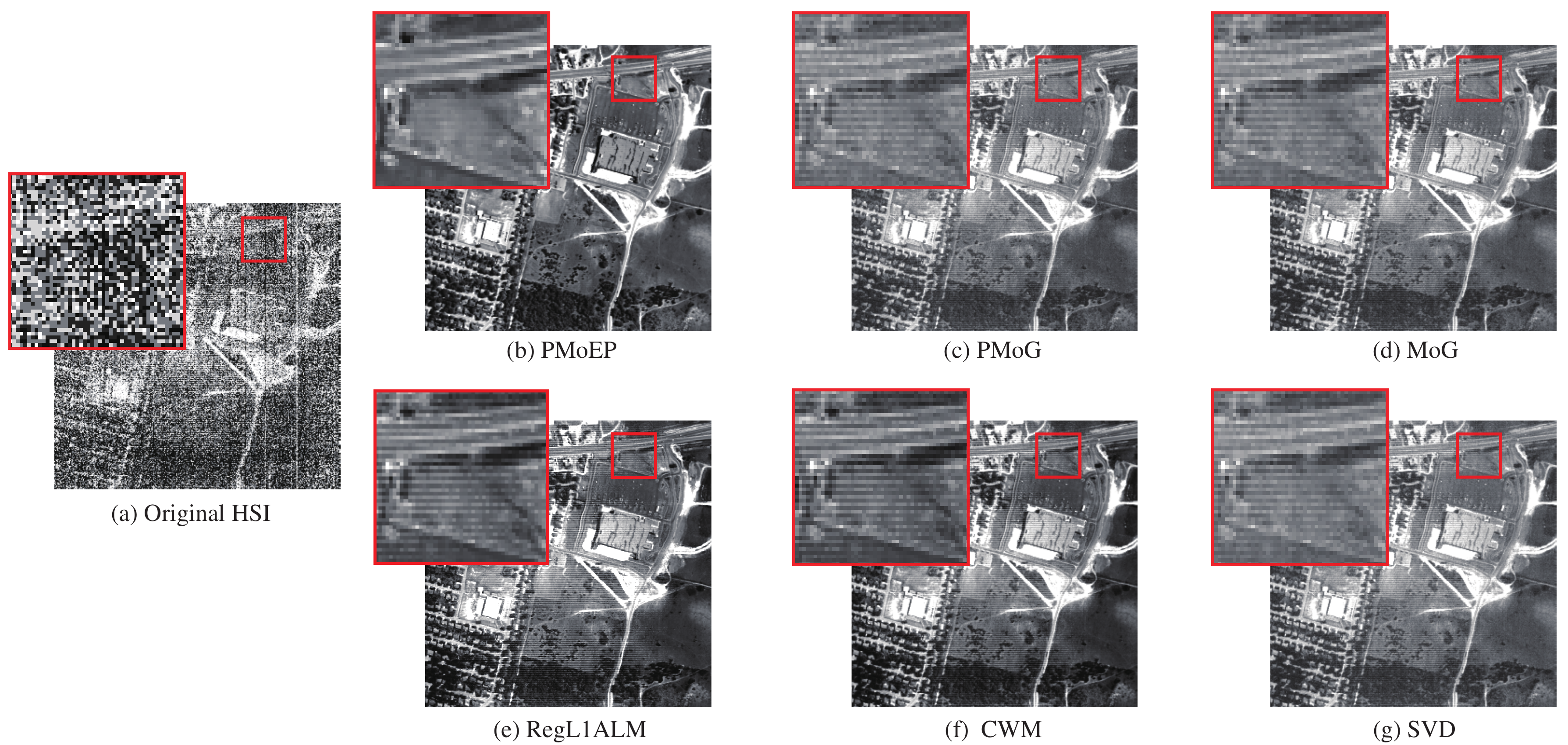}
	\caption{ Restoration results of band 207 in Urban data set: (a) original bands. (b)-(g) reconstructed bands by PMoEP, PMoG, MoG, RegL1ALM, CWM and SVD.}\label{207bandRec}
\end{figure}

\begin{figure}
	\centering
	\includegraphics[width=1.0\linewidth]{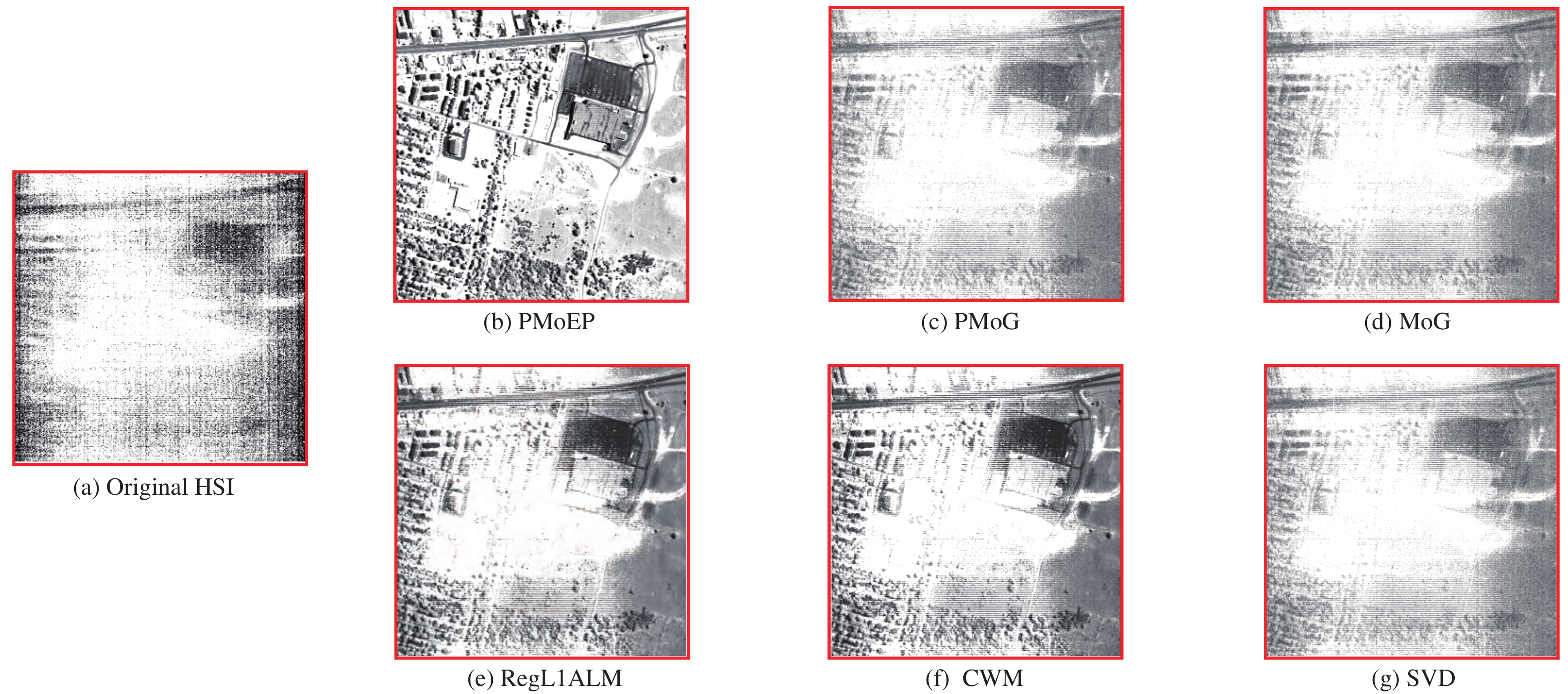}
	\caption{Restoration results of band 107 in Urban data set: (a) original bands. (b)-(g) reconstructed bands by PMoEP, PMoG, MoG, RegL1ALM, CWM and SVD.}\label{107bandRec}
\end{figure}

\begin{figure}
	\centering
	\includegraphics[width=1.0\linewidth]{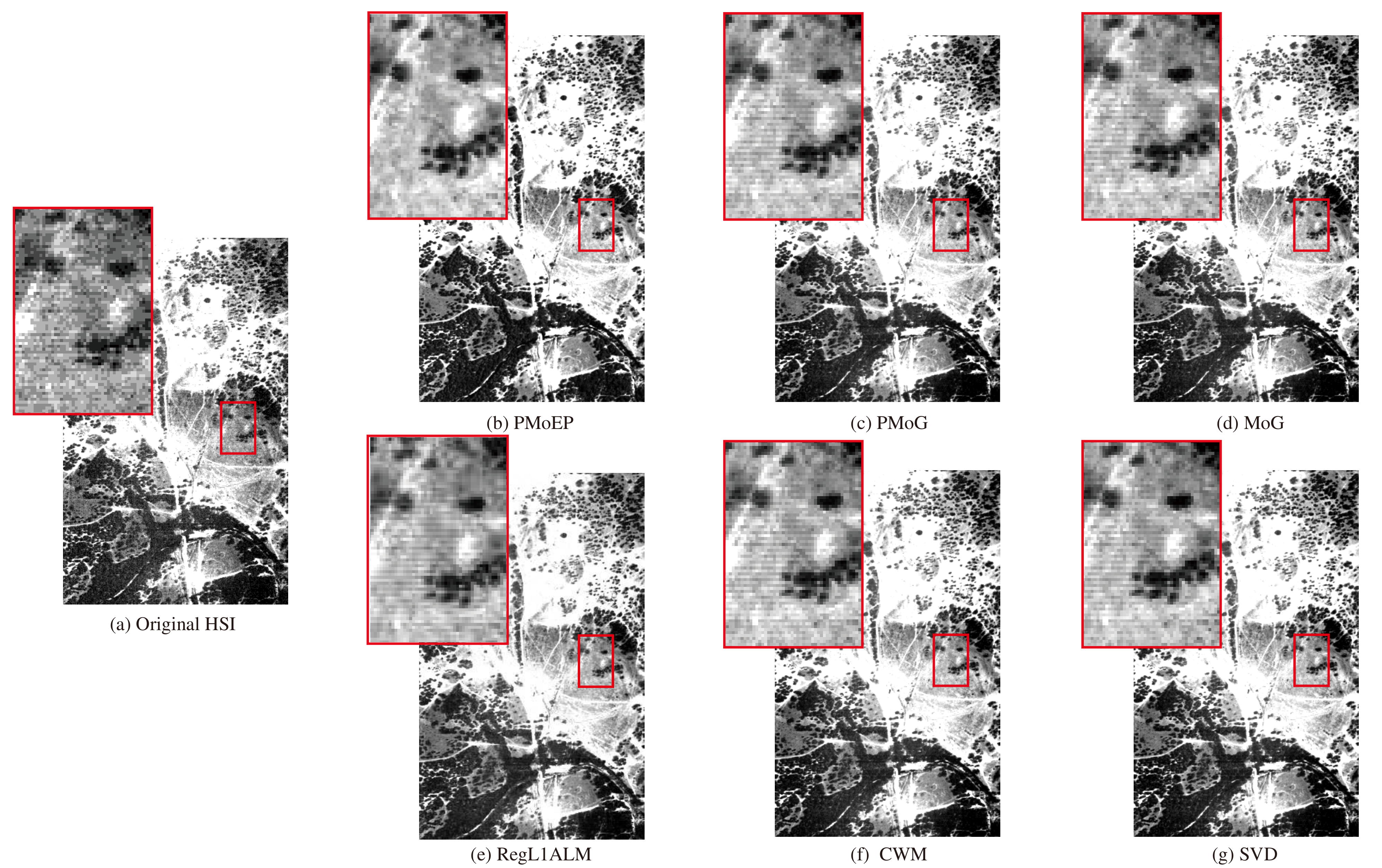}
	\caption{ Restoration results of band 152 in Terrain data set: (a) original bands. (b)-(g) reconstructed bands by PMoEP, PMoG, MoG, RegL1ALM, CWM and SVD.}\label{152Rec}
\end{figure}

\subsection{Hyperspectral Image Restoration}
In this section, we evaluate the performance of our proposed PMoEP method on hyperspectral image restoration problem. Two real hyperspectral image (HSI) data sets\footnote{http://www.tec.army.mil/hypercube.}  were used.

The first dataset is Urban HSI data. This dataset contains $210$ bands, each of which  is $307\times307$, and some bands are seriously polluted by atmosphere and water and corrupted by noises with complex structures, as shown in Fig. \ref{intro_fig}. We reshape each band as a vector, and stack all the vectors into a matrix, resulting in the final data matrix with size $94249\times210$. The second one is the Terrain dataset. The original images are of size $500\times307\times210$. We use all the bands in our experiments and thus generate a $153500\times210$ data matrix. Therefore, we get two data matrices used to test our methods. All the competing methods were implemented, except DW method which encounters the `out of memory' problem.

\begin{figure}
	\centering
	\includegraphics[width=1.0\linewidth]{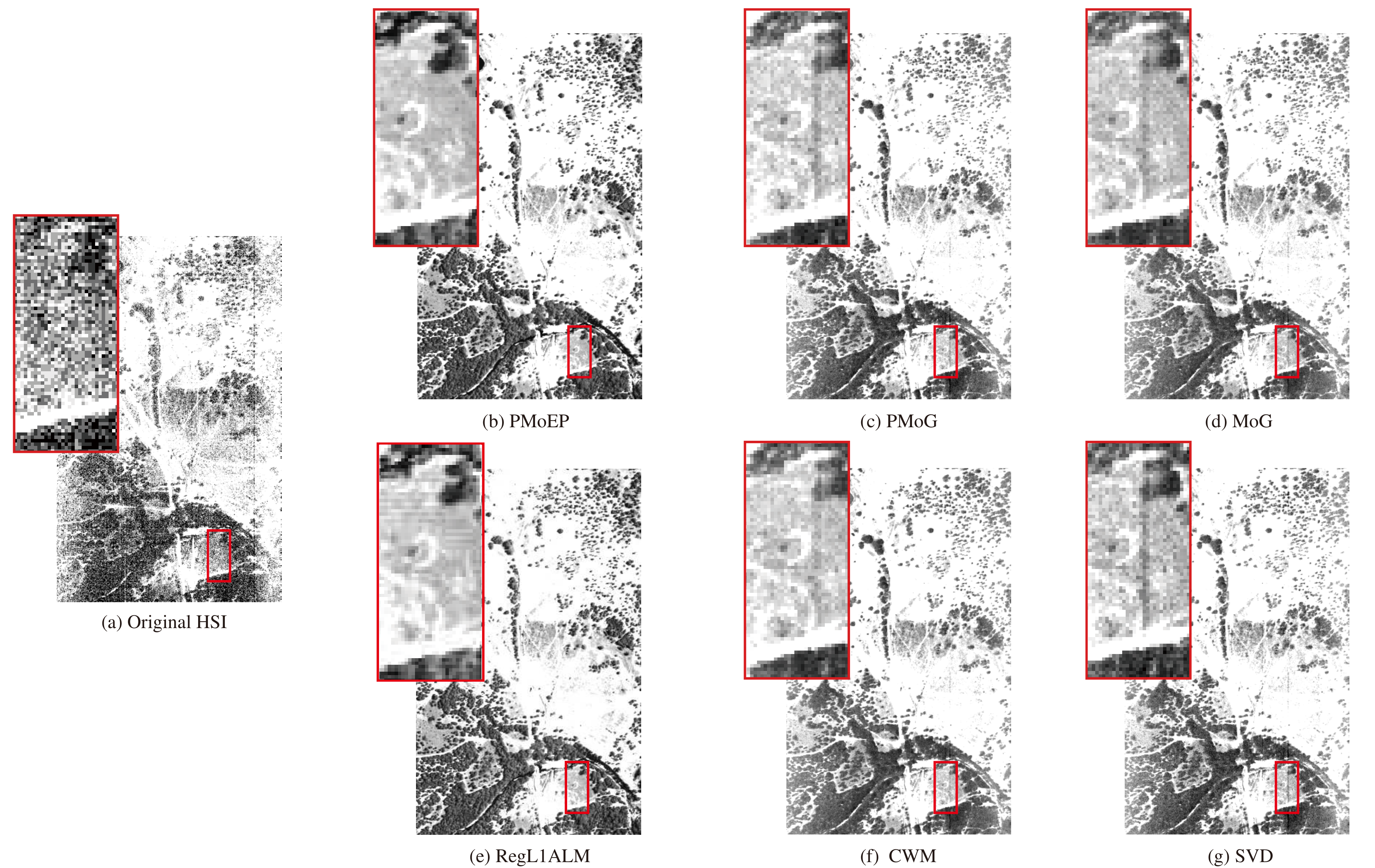}
	\caption{ Restoration results of band 206 in Terrain data set: (a) original bands. (b)-(g) reconstructed bands by PMoEP, PMoG, MoG, RegL1ALM, CWM and SVD.}\label{206Rec}
\end{figure}

\begin{figure}
	\centering
	\includegraphics[width=1.0\linewidth]{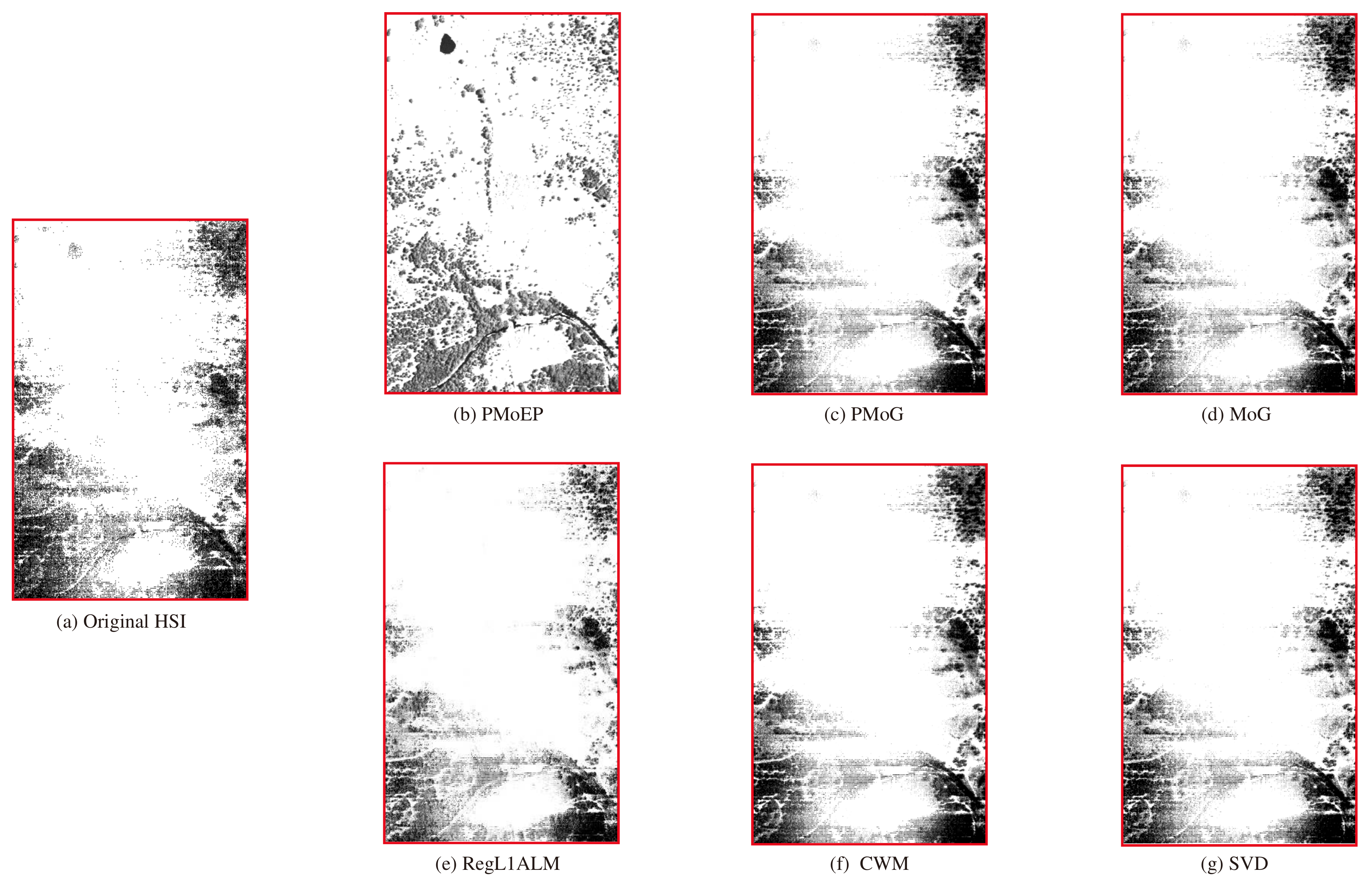}
	\caption{ Restoration results of band 139 in Terrain data set: (a) original bands. (b)-(g) reconstructed bands by PMoEP, PMoG, MoG, RegL1ALM, CWM and SVD.}\label{139Rec}
\end{figure}

\begin{figure*}
	\centering
	\includegraphics[width=0.8\linewidth]{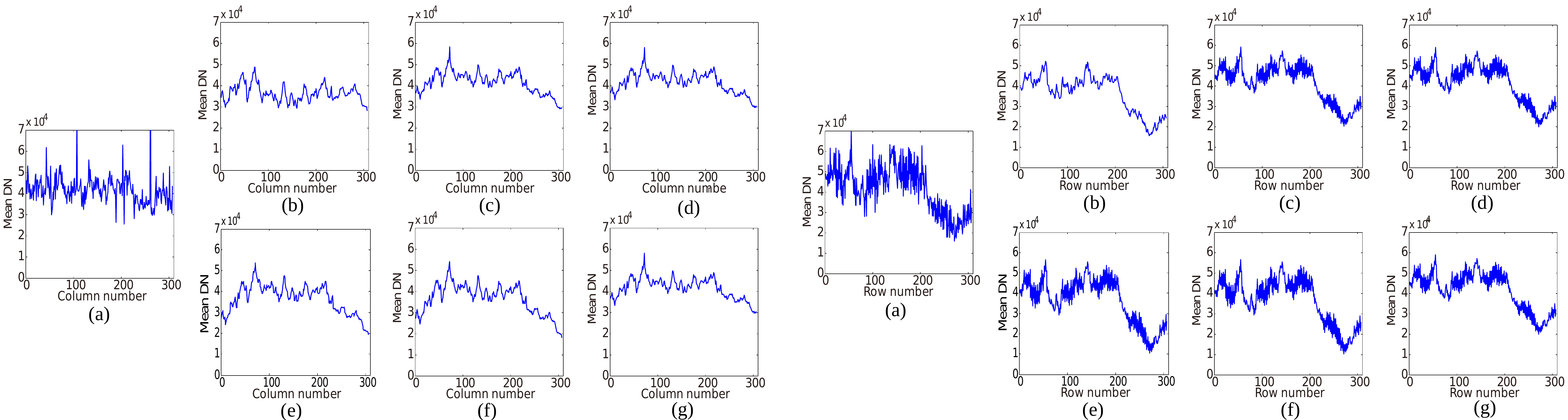}
	\caption{ Vertical (left) and Horizontal (right) mean profiles of band 207 in the Urban data set: (a) original, (b) PMoEP, (c) PMoG, (d) MoG, (e) RegL1ALM, (f) CWM, (g) SVD.}\label{M207}
\end{figure*}

\begin{figure*}
	\centering
	\includegraphics[width=0.8\linewidth]{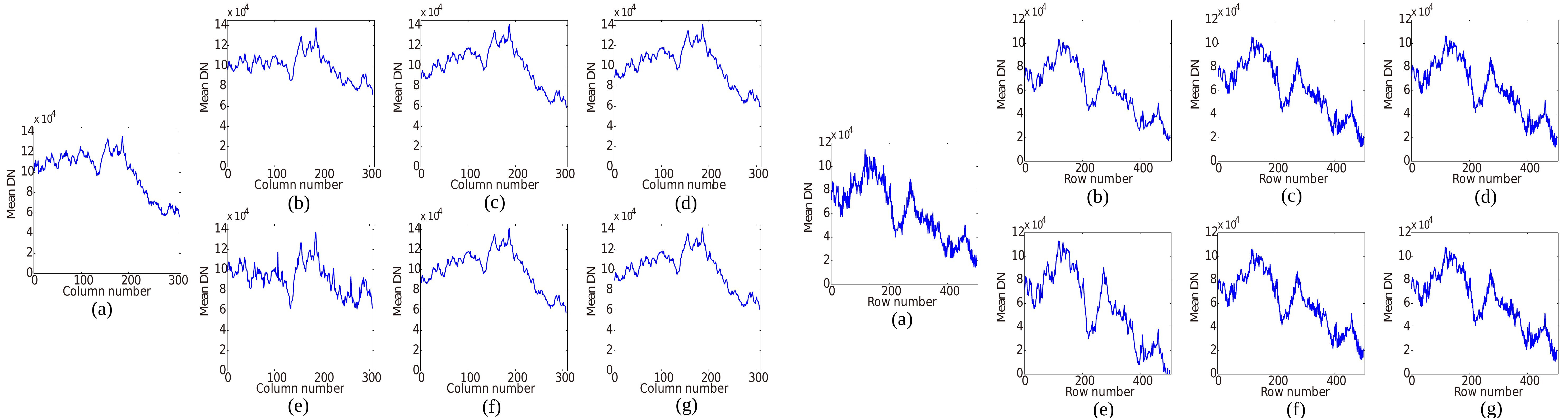}
	\caption{Vertical (left) and Horizontal (right) mean profiles of band 152 in the Terrain data set: (a) original, (b) PMoEP, (c) PMoG, (d) MoG, (e) RegL1ALM, (f) CWM, (g) SVD.}\label{M152}
\end{figure*}

\begin{figure*}
	\centering
	\includegraphics[width=0.8\linewidth]{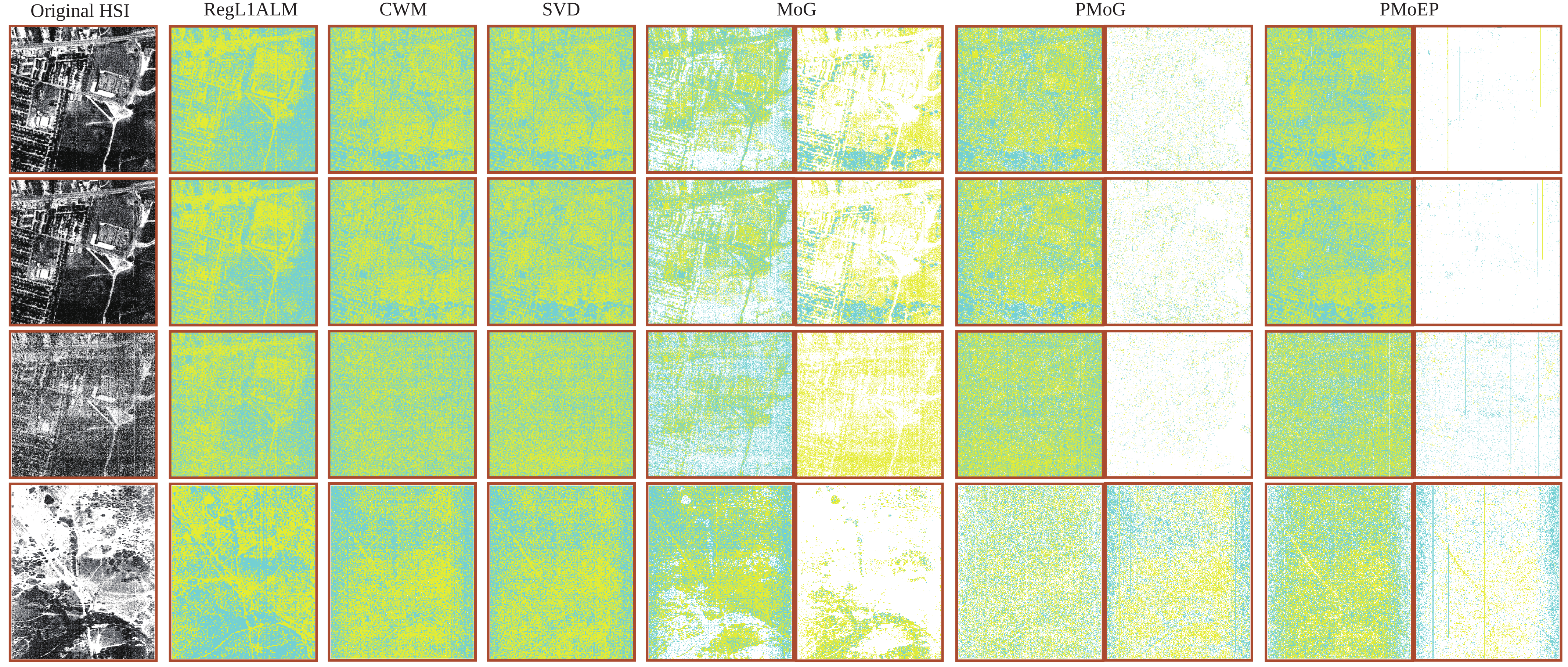}
	\caption{From top to bottom, band 204, 206, 207 of Urban, band 208 of Terrain. From left to right: original bands, and extracted noise by RegL1ALM, CWM, SVD, MoG, PMoG and PMoEP. The noises with positive and negative values are depicted in yellow and blue, respectively. This figure should be viewed in color and the details are better seen by zooming on a computer screen.}\label{UrbanNoise}
\end{figure*}

The reconstructed hyperspectral images of bands $103$, $206$, $207$ and $107$ in Urban dataset and bands $152$, $206$ and $139$ in Terrain dataset are shown in Fig. \ref{103bandRec}$-$\ref{107bandRec} and  Fig. \ref{152Rec}$-$\ref{139Rec}, respectively. For easy observation, an area of interest is amplified in the restored images obtained by all the competing methods. It can be easily seen from the figures that for some bands containing evident stripes and deadlines, the image restored by the proposed PMoEP method is clean and smooth, while the results obtained by the other competing ones contain evident stripe area. In addition, as is demonstrated in Fig. \ref{107bandRec} and  Fig. \ref{139Rec}, the PMoEP method can effectively recover the seriously polluted bands, while the other methods failed on them. These results show that our proposed PMoEP method can not only remove complicated noises embedded in HSI, but also can perform robust in the presence of extreme outlier cases like in Fig. \ref{107bandRec} and Fig. \ref{139Rec}.

Then we give more quantitative comparison by showing the vertical mean profiles and horizontal mean profiles of band 207 in Urban dataset and band 152 in Terrain dataset before and after reconstruction in Fig. \ref{M207} and Fig. \ref{M152}. The horizontal axis of Fig. \ref{M207} represents the column (left) and row (right) number, and the vertical axis represents the mean DN value of each column (left) and row (right). It is easy to observe that the curves in Fig. \ref{M207}(a) and \ref{M152}(a) (right) have drastic fluctuations for the original image. This is deviated from the prior knowledge that the adjacent bands should possess similar shapes since they are captured under relatively similar sensor settings. After the reconstruction, the fluctuations in vertical direction have been reduced by most of the methods. While in the horizontal direction (see Fig. \ref{M207} (right) and Fig. \ref{M152} (right)), the PMoEP method provides evidently smoother curves, which indicates that the stripes in the horizontal direction have been removed more effectively by our method. The results are consistent with the recovered HSIs in Fig. \ref{207bandRec} and Fig. \ref{152Rec}.

The better performance of PMoEP over other methods is due to its more powerful ability in noise modeling. Specifically, as depicted in Fig. \ref{UrbanNoise}, PMoEP can more properly extract noise information from the corrupted images with physical meanings, such as sparse strips, sparse deadlines, and dense Gaussian noise, while other competing methods fail to do so.

\begin{figure*}\label{ForDet}
	\centering
	\includegraphics[width=0.85\linewidth]{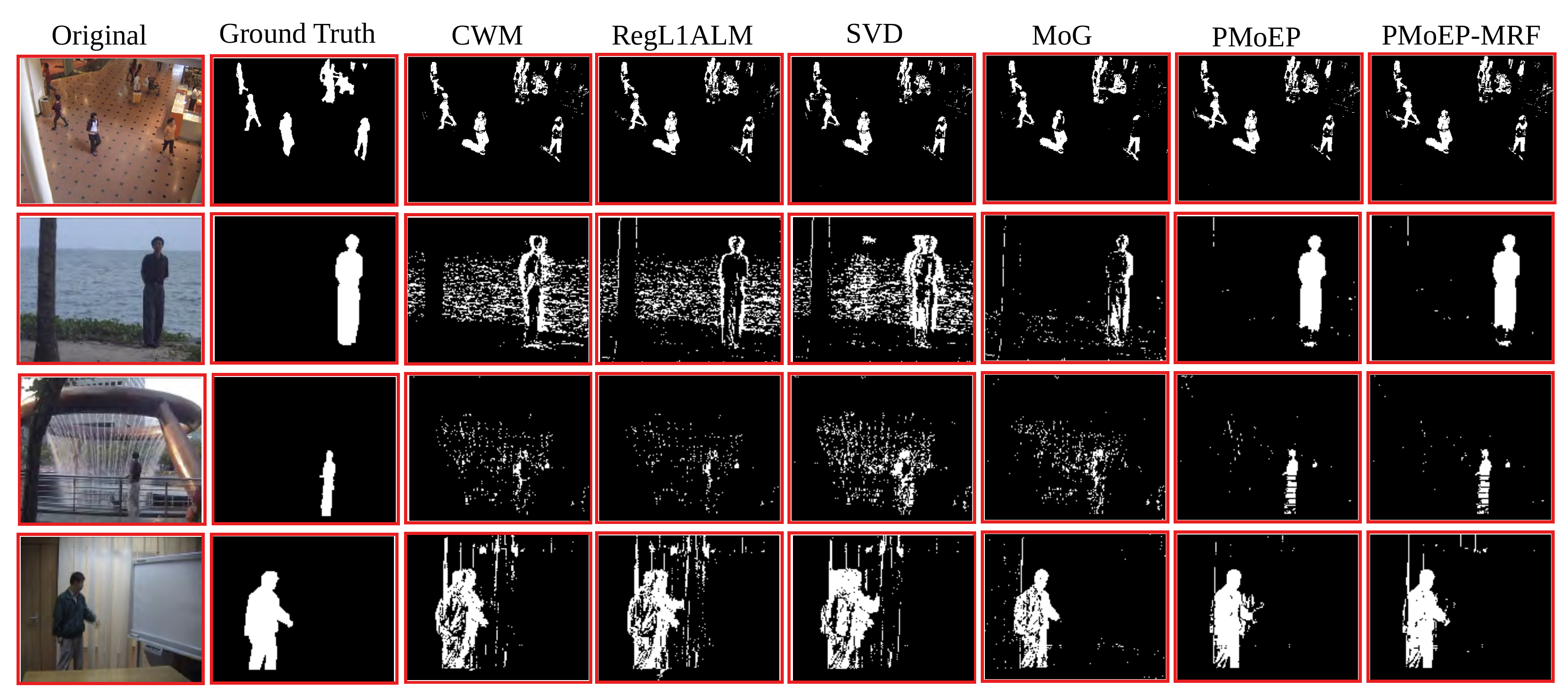}
	\caption{Foreground Detection results of different methods on sample frames.}\label{watersurface}
\end{figure*}

\subsection{Background Subtraction}
In this section, we evaluate the performance of our proposed methods on background subtraction problem. The background subtraction from a video sequence captured by a static camera can be modeled as a low-rank matrix analysis problem~\cite{wright2009robust}. All the nine standard video sequences\footnote{http://perception.i2r.a-star.edu.sg/bk\_model/bk\_index.html} provided by Li et.al~\cite{li2004statistical} were adopted in our evaluation, including simple and complex background. Ground truth foreground regions of 20 frames were provided for each sequence.

We compared our PMoEP and PMoEP-MRF methods with the state-of-the-art LRMF methods: SVD, RegL1ALM, CWM and MoG methods. To conduct the experiments, we first ran each method on each video sequence to estimate the background. Then we obtained the recovered foreground by calculating the absolute values of the difference between the original frame and the estimated background. For MoG, PMoEP and PMoEP-MRF methods, we obtained the foreground by selecting the noise component with largest variance.

For quantitative evaluation, we first introduce some evaluation indices. We measure the recovery accuracy of the support in the foreground by comparing the true support $S$ with the detected support $\tilde{S}$. We regard it as a classification problem and thus can evaluate the results using precision and recall, which are defined as:
\begin{displaymath}
precision = \frac{TP}{TP+FP},~~~ recall = \frac{TP}{TP+FN},
\end{displaymath}
where $TP$, $FP$, $TN$ and $FN$ represent the numbers of true positive,
false positive, true negative and false negative, respectively.
For simplicity, we adopt $F\textrm{-}measure$ that combines
precision and recall together:
\begin{displaymath}
F\textrm{-}measure = 2\times \frac{precision\times recall}{precision+recall}.
\end{displaymath}
The higher F\textrm{-}measure value means the
better recovery accuracy of the support. Additionally, the
recovered support $\tilde{S}$ is obtained by thresholding the recovered foreground  $E$ with a threshold value that gives the maximal F\textrm{-}measure. For all competing methods, we adopt two initialization strategies, namely, random and SVD. Then we report the best result among the two initializations. The results are summarized in Table \ref{simuTab2}.

	\begin{table}[htp]
		\caption{\label{simuTab2}Performance evaluation on Video data. The best and second best results for each video dataset are highlighted in bold and in italic bold, respectively.}
		\begin{center}
			{\normalsize
				\scalebox{0.7}[0.75]{
					\begin{tabular}{cccccccc}
						\toprule[2pt]
						Video  & SVD & RegL1ALM & CWM & MoG & PMoEP &  PMoEP-MRF \\
						\midrule[2pt]
						\multicolumn{7}{c}{$F\textrm{-}measure$}\\ \hline
						Campus       & 0.4716  & {\bf{0.5308}} & \textbf{\emph{0.5301}} &  0.4633  & 0.5065 & 0.5115\\
						Lobby        & 0.7623  & 0.7679 & \textbf{\emph{0.7681}} &   {\bf{0.7724}}  & 0.7650 & 0.7444\\
						ShoppingMall & 0.6990  & \textbf{\emph{0.7138}} & {\bf{0.7173}}  &   0.6387  &  0.7037 & 0.7015\\
						Bootstrap    & 0.6234  & {\bf{0.6749}}  & 0.6533 &   0.4234  &  0.6404 &  \textbf{\emph{0.6635}} \\
						Hall         & 0.4104  & 0.4659   &  0.4624  &  0.4523  &  \textbf{\emph{0.5372}}  & {\bf{0.5438}}\\
						Curtain      & 0.5273  & 0.5342  &  0.5316  &  0.7869  & {\bf{0.7895}} & \textbf{\emph{0.7888}}\\
						Fountain     & 0.4989  & 0.5298 & 0.5262 &   0.5782  & \textbf{\emph{0.6843}} &  {\bf{0.7295}}\\
						WaterSurface & 0.3416  & 0.2840 & 0.2920 &   0.5979  & \textbf{\emph{0.8515}} & {\bf{0.8651}}\\
						Escalator    & 0.2675  & 0.2998 & 0.2972 &   0.2675  & \textbf{\emph{0.3255}} & {\bf{0.3408}}\\
						\hline
						Average      & 0.5113  & 0.5334 & 0.5309 &   0.5534  & \textbf{\emph{0.6448}} & {\bf{0.6543}}\\
						\bottomrule[2pt]
					\end{tabular}}
				}
			\end{center}
		\end{table}

From Table \ref{simuTab2}, it can be easily seen that our proposed PMoEP and PMoEP-MRF methods outperform other methods in the sequences of Hall, Curtain, Fountain, WaterSurface and Escalator, of which the background is with complex shapes. For the sequences with simple background, including Bootstrap, ShoppingMall, Campus and Lobby, the performances of all the methods are almost the same. On average, the PMoEP method achieves the second best performance. Compared with the PMoEP method, the PMoEP-MRF method slightly improves the average performance due to the modeling of spatial and temporal smoothness prior knowledge under foreground using Markov random field.

The better performance of PMoEP and PMoEP-MRF methods can be visually shown in Fig. 15. It can be easily seen from the figure that the proposed PMoEP and PMoEP-MRF can perform comparably well as other methods in simple foreground cases, while evidently better in much complicated scenarios, e.g., videos with dynamic background.

\section{Conclusions}
In this paper, we model the noise of the LRMF problem as a Mixture of Exponential Power (MoEP) distributions and proposes a penalized MoEP (PMoEP) model by combining the penalized likelihood method with the MoEP distributions. Moreover, by facilitating the local continuity of noise components along both space and time of a video, we embed Markov random field into PMoEP and then propose the PMoEP-MRF model. Compared with the current LRMF methods, our PMoEP method performs better in a wide variety of synthetic and real complex noise scenarios including face modeling, hyperspectral image restoration, and background subtraction applications. Additionally, our methods are capable of automatically learning the number of components from data, and thus can be used to deal with more complex applications. In the future, we'll attempt to extend the noise modeling methodology under PMoEP to more computer vision and machine learning tasks, e.g., the high-order low rank tensor factorization problems.

\section*{Acknowledgements}
This research was supported by 973 Program of China
with No.3202013CB329404, the NSFC projects with No.11131006, 91330204 and 61373114.

\appendices
\section{Proof of Theorem 1}
\begin{proof}
	(i) First, we calculate that
	\footnotesize
	\begin{eqnarray}
	\begin{split}
	l_{P}^{\textsc{C}}(\mathbf{\Theta})-l_{P}^{C}(\mathbf{\Theta}^{(t)})
	&=l(\mathbf{\Theta})-l(\mathbf{\Theta}^{(t)})+P(\pi^{(t)};\lambda)-P(\pi;\lambda)\nonumber \\
	 &=\log{\sum_{\mathbf{Z}}\mathbb{P}(\mathbf{Z}|\mathbf{E};\mathbf{\Theta}^{(t)})\frac{\mathbb{P}(\mathbf{E}|\mathbf{Z};\mathbf{\Theta})\mathbb{P}(\mathbf{Z};\mathbf{\Theta})}{\mathbb{P}(\mathbf{Z}|\mathbf{E};\mathbf{\Theta}^{(t)})}}\\
	&~~~-\log{\mathbb{P}(\mathbf{E};\mathbf{\Theta}^{(t)})}+P(\pi^{(t)};\lambda)-P(\pi;\lambda)\nonumber \\
	&\geq \sum_{\mathbf{Z}}\mathbb{P}(\mathbf{Z}|\mathbf{E};\mathbf{\Theta}^{(t)})\log{\frac{\mathbb{P}(\mathbf{E}|\mathbf{Z};\mathbf{\Theta})\mathbb{P}(\mathbf{Z};\mathbf{\Theta})}{\mathbb{P}(\mathbf{Z}|\mathbf{E};\mathbf{\Theta}^{(t)})}}\\
	&~~~-\log{\mathbb{P}(\mathbf{E};\mathbf{\Theta}^{(t)})}+P(\pi^{(t)};\lambda)-P(\pi;\lambda)\nonumber \\
	 &=\sum_{\mathbf{Z}}\mathbb{P}(\mathbf{Z}|\mathbf{E};\mathbf{\Theta}^{(t)})\log{\frac{\mathbb{P}(\mathbf{E}|\mathbf{Z};\mathbf{\Theta})\mathbb{P}(\mathbf{Z};\mathbf{\Theta})}{\mathbb{P}(\mathbf{Z}|\mathbf{E};\mathbf{\Theta}^{(t)})\mathbb{P}(\mathbf{E};\mathbf{\Theta}^{(t)})}}\\
	&~~~+P(\pi^{(t)};\lambda)-P(\pi;\lambda).\nonumber
	\end{split}
	\end{eqnarray}
	Let $\Omega(\mathbf{\Theta}|\mathbf{\Theta}^{(t)})=\sum_{\mathbf{Z}}\mathbb{P}(\mathbf{Z}|\mathbf{E};\mathbf{\Theta}^{(t)})\log{\frac{\mathbb{P}(\mathbf{E}|\mathbf{Z};\mathbf{\Theta})\mathbb{P}(\mathbf{Z};\mathbf{\Theta})}
		{\mathbb{P}(\mathbf{Z}|\mathbf{E};\mathbf{\Theta}^{(t)})\mathbb{P}(\mathbf{E};\mathbf{\Theta}^{(t)})}}$, then
	\begin{equation}\label{lemma1}
	l_{P}^{C}(\mathbf{\Theta})\geq l_{P}^{C}(\mathbf{\Theta}^{(t)}) + \Omega(\mathbf{\Theta}|\mathbf{\Theta}^{(t)})+P(\pi^{(t)};\lambda)-P(\pi;\lambda)\nonumber.
	\end{equation}
	\normalsize
	\\
	(ii) In the M step of Algorithm 1, it is obvious that
	\footnotesize
	\begin{eqnarray}
	\mathbf{\Theta}^{(t+1)}\!&=&\! \underset{\mathbf{\Theta}}{\!\arg\max}\left\{\sum_{\mathbf{Z}}\mathbb{P}(\mathbf{Z}|\mathbf{E};\mathbf{\Theta}^{(t)})\log{\mathbb{P}(\mathbf{E},\mathbf{Z};\mathbf{\Theta})}\!-\!P(\pi;\lambda)\right\}\nonumber\\
	 \!&=&\!\underset{\mathbf{\Theta}}{\!\arg\max}\left\{\sum_{\mathbf{Z}}\mathbb{P}(\mathbf{Z}|\mathbf{E};\mathbf{\Theta}^{(t)})\frac{\log{\mathbb{P}(\mathbf{E},\mathbf{Z};\mathbf{\Theta})}}
	{\log{\mathbb{P}(\mathbf{E},\mathbf{Z};\mathbf{\Theta}^{(t)})}}\!-\!P(\pi;\lambda)\right\}\nonumber\\
	 \!&=&\!\underset{\mathbf{\Theta}}{\!\arg\max}\left\{\Omega(\mathbf{\Theta}|\mathbf{\Theta}^{(t)})+P(\pi^{(t)};\lambda)-P(\pi;\lambda)\right\}.\nonumber
	\end{eqnarray}
	\normalsize
	Thus, we have
	\begin{eqnarray}
	\begin{split}
	&\Omega(\mathbf{\Theta}^{(t+1)}|\mathbf{\Theta}^{(t)})+P(\pi^{(t)};\lambda)-P(\pi^{(t+1)};\lambda)\\
	&\geq
	\Omega(\mathbf{\Theta}^{(t)}|\mathbf{\Theta}^{(t)})+P(\pi^{(t)};\lambda)-P(\pi^{(t)};\lambda)=0
	\end{split}
	\end{eqnarray}
	Then, we can easily derive that
	\begin{equation}\label{inequ}
	l_{P}^{C}(\mathbf{\Theta}^{(t+1)})\geq l_{P}^{C}(\mathbf{\Theta}^{(t)}).\nonumber
	\end{equation}
	Based on (\ref{inequ}), the sequence $\{l_{P}^{G}(\mathbf{\Theta}^{(t)})\}_{t=1}^{\infty}$ is nondecreasing and bounded above. Therefore, there exits a constant
	$l^{\star}$ such that
	\begin{equation}
	\lim_{t\rightarrow \infty}l_{P}^{C}(\mathbf{\Theta}^{(t)}) = l^{\star}.\nonumber
	\end{equation}
\end{proof}

\section{Exponential Power Distribution}
\subsection{Three different forms of Exponential Power Distribution}
The Exponential Power Distribution ($\mu=0$) has the following three equivalent forms:
\begin{equation}\label{form1}
f_{p}(x;0,\sigma)=\frac{1}{2\sigma p^{\frac{1}{p}}\Gamma(1+\frac{1}{p})}\exp\left\{-\frac{|x|^{p}}{p\sigma^{p}}\right\}.\nonumber
\end{equation}
Let $\tau = (p\sigma^{p})^{\frac{1}{p}}$, then
\begin{equation}\label{form2}
f_{p}(x;0,\tau)=\frac{1}{2\tau\Gamma(1+\frac{1}{p})}\exp\left\{-|\frac{x}{\tau}|^{p}\right\}.\nonumber
\end{equation}
Let $\eta = \frac{1}{\tau^{p}}$, then
\begin{equation}\label{form3}
f_{p}(x;0,\eta)=\frac{\eta^{\frac{1}{p}}}{2\Gamma(1+\frac{1}{p})}\exp\left\{-\eta|x|^{p}\right\}.\nonumber
\end{equation}
Noting that $\Gamma(1+\frac{1}{p})=\frac{1}{p}\Gamma(\frac{1}{p})$, then we can represent the above three forms in equivalent forms.

\subsection{Draw Samples from Exponential Power Distribution}
The second form of exponential power distribution is
\begin{equation}\label{form2}
f_{p}(x;0,\tau)=\frac{1}{2\tau\Gamma(1+\frac{1}{p})}\exp\left\{-|\frac{x}{\tau}|^{p}\right\}.\nonumber
\end{equation}
Sampling from the exponential power distribution contains two cases: $p\geq 1$ and $0<p<1$.
\subsubsection{case 1: $p\geq 1$}
We adopt the method proposed in \cite{chiodi1995generation,marsaglia1964convenient,mineo2005software}.
\subsubsection{case 2: $0<p<1$}
When $0<p<1$, the method proposed in \cite{polson2014bayesian} is used. We sample the distribution in two steps:
\begin{equation}\label{samplew}
(w|p) \sim \frac{1+p}{2}Ga(2+\frac{1}{p},1) + \frac{1-p}{2}Ga(1+\frac{1}{p},1),
\end{equation}
\begin{equation}\label{samplebeta}
(\beta|\tau,w,p) \sim \frac{1}{\tau w^{\frac{1}{p}}}\left\{1-|\frac{\beta}{\tau w^{\frac{1}{p}}}|\right\}_{+},
\end{equation}
where $w$ is a intermediate variable. (\ref{samplew}) can be sampled directly but (\ref{samplebeta}) is difficult. Therefore, we
adopt the slice sampling strategy in \cite{bishop2006pattern}.

%
%

\ifCLASSOPTIONcaptionsoff
\newpage
\fi

\bibliographystyle{ieee}
\bibliography{mybibfile}

\end{document}